\documentclass{article}
\usepackage[margin=1.3in]{geometry}

\usepackage[utf8]{inputenc}

\usepackage{graphicx}
\usepackage{caption}
\usepackage{subcaption}

\usepackage{amsmath,amsthm,amsfonts} 
\usepackage{bm} 
\usepackage{setspace}

\usepackage{algorithm}
\usepackage{algorithmic}

\usepackage[mathscr]{eucal}

\usepackage{amssymb}

\usepackage{mathtools}

\usepackage[dvipsnames,svgnames,x11names]{xcolor}

\definecolor{frenchblue}{rgb}{0.0, 0.45, 0.73}
\definecolor{forestgreen(web)}{rgb}{0.13, 0.55, 0.13}

\usepackage[bookmarksopen=true]{hyperref}
\hypersetup{breaklinks, colorlinks=true,linkcolor=frenchblue, citecolor=forestgreen(web)}

\usepackage[nameinlink]{cleveref}

\usepackage{bookmark}

\usepackage[square,sort,comma,numbers]{natbib}

\usepackage{mathrsfs}

\newcommand{\spro}{\begin{proof}}
\newcommand{\fpro}{\end{proof}}

\usepackage{thmtools}

\makeatletter
\def\thmt@refnamewithcomma #1#2#3,#4,#5\@nil{%
  \@xa\def\csname\thmt@envname #1utorefname\endcsname{#3}%
  \ifcsname #2refname\endcsname
    \csname #2refname\expandafter\endcsname\expandafter{\thmt@envname}{#3}{#4}%
  \fi
}
\makeatother

\declaretheorem[numberwithin=section, name=Theorem,Refname={Theorem,Theorems}]{theo}
\declaretheorem[numberwithin=section, name=Lemma,Refname={Lemma,Lemmas}]{lemm}

\declaretheorem[numberwithin=section, name=Definition,Refname={Definition,Definitions}]{defi}

\declaretheorem[numberwithin=section, name=Assumption,Refname={Assumption,Assumptions}]{assm}

\declaretheorem[numberwithin=section, name=Definition,Refname={Definition,Definitions}]{definition}

\declaretheorem[numberwithin=section, name=Remark,Refname={Proposition,Propositions}]{remark}

\DeclareMathOperator{\Expe}{\mathbb{E}}
\DeclareMathOperator{\Real}{\mathbb{R}}

\newcommand{\normtv}[1]{\norm{#1}_{\rm TV}}

\newcommand{\Demo}{\mathcal{D}}

\newcommand{\pih}{\widehat{\pi}}
\newcommand{\pis}{\pi_{\rm E}}

\newcommand{\dist}{\rho_{\pis}}

\newcommand{\pir}{\widehat{\pi}^{\rm RBC}}
\newcommand{\fhat}{\widehat{f}}

\newcommand{\Fcal}{\mathcal{F}}
\newcommand{\Dist}{\mathscr{D}}

\newcommand{\Dcal}{\mathcal{D}}

\DeclarePairedDelimiter\abs{\lvert}{\rvert}%
\DeclarePairedDelimiter\norm{\lVert}{\rVert}%

\makeatletter
\let\oldabs\abs
\def\abs{\@ifstar{\oldabs}{\oldabs*}}
\let\oldnorm\norm
\def\norm{\@ifstar{\oldnorm}{\oldnorm*}}
\makeatother

\newcommand{\obs}{\mathbf{s}}
\newcommand{\acts}{\mathbf{a}}

\newcommand{\rhoExp}{\rho_{\pis}}

\newcommand{\Sset}{\mathcal{S}}
\newcommand{\Aset}{\mathcal{A}}
\newcommand{\Tran}{\mathsf{P}}

\newcommand{\E}{\Expe}

\usepackage{varwidth}

\usepackage{enumitem}
\setlist{leftmargin=3.5mm}

\usepackage{microtype}
\usepackage{graphicx}
\usepackage{booktabs} 

\usepackage{amsmath}
\usepackage{amssymb}
\usepackage{mathtools}
\usepackage{amsthm}

\title{
\textbf{Robust Imitation Learning from Corrupted Demonstrations
}}

\usepackage{authblk}

\author[1]{Liu Liu} 

\author[2]{Ziyang Tang}
   
\author[1]{Lanqing Li}

\author[1]{Dijun Luo}

\affil[1]{Tencent AI Lab}
\affil[2]{The University of Texas at Austin
}
\affil[ ]{\texttt {\{leonliuliu,lanqingli,dijunluo\}@tencent.com},
\texttt{ztang@cs.utexas.edu}}

\date{}

\begin{document}

\maketitle





\begin{abstract}
We consider offline Imitation Learning  
from \emph{corrupted demonstrations} where a constant fraction of data can be noise or even arbitrary outliers. Classical approaches such as Behavior Cloning assumes that demonstrations are collected by an presumably optimal expert, hence
may fail drastically when learning from corrupted demonstrations.
We propose a novel robust algorithm by minimizing a Median-of-Means (MOM) objective which guarantees the accurate estimation of policy, even in the presence of \emph{constant} fraction of outliers. 
Our theoretical analysis shows that our robust method in the corrupted setting enjoys nearly the same error scaling and sample complexity guarantees as the classical Behavior Cloning in the expert demonstration setting.
Our experiments on continuous-control benchmarks validate that  our method exhibits the predicted robustness and effectiveness,
and achieves competitive results compared to existing imitation learning methods.
\end{abstract}

\section{Introduction}
\label{sec:intro}









Recent years have witnessed the success of using  autonomous agent
to learn and adapt to complex tasks and environments in a range of applications such as playing games~\citep[e.g.][]{mnih2015human, silver2018general, vinyals2019grandmaster}, autonomous driving~\citep[e.g.][]{kendall2019learning, bellemare2020autonomous}, robotics~\citep{haarnoja2017SoftQ}, medical treatment~\citep[e.g.][]{RLHealthCare} and recommendation system and advertisement~\citep[e.g.][]{li11unbiased, thomas17predictive}.

Previous success for sequential decision making often requires two key components: 
(1) a careful design reward function that can provide the supervision signal during learning and 
(2) an unlimited number of online interactions with the real-world environment (or a carefully designed simulator) to query new unseen region.
However, in many scenarios, both components are not allowed.
For example, it is hard to define the reward signal in uncountable many extreme situations in autonomous driving \cite{survey_DRL_autonomousdriving};
and it is dangerous and risky to directly deploy a learning policy on human to gather information in autonomous medical treatment \citep{RLHealthCare}.
Therefore an \emph{offline} sequential decision making algorithm without reward signal is in demand.

Imitation Learning (IL) \cite{Pieter2018algorithmic}  offers an elegant way to train intelligent agents for complex task without the knowledge of reward functions. 
In order to guide intelligent agents to correct behaviors, it is crucial to have high quality expert demonstrations.
The well-known imitation learning algorithms such as Behavior Cloning (BC, \cite{Pomerleau1988NIPS_BC}) or
Generative Adversarial Imitation Learning (GAIL, \cite{ho2016GAIL})
require that the demonstrations given for training are all \emph{presumably optimal} and it aims to learn the optimal policy from expert demonstration data set.
More specifically, BC only
uses offline demonstration data without any interaction with
the environment, whereas GAIL requires online interactions.

However in real world scenario,
since the demonstration is often collected from human, we cannot guarantee that \emph{all} the demonstrations we  collected have high quality.
This has been addressed in a line of research \cite{wu2019imperfect, tangkaratt2020VILD, tangkaratt2021RobustImitation, brown2019TREX, NoisyBC2020}.
An human expert can  make mistakes by accident or due to the hardness of a complicated scenario (e.g., medical diagnosis).
Furthermore, even an expert demonstrates a successful behavior, the recorder or the recording system can have a chance to contaminate the data by accident or on purpose \citep[e.g.][]{neff2016automation, eykholt2018robust, Zhu2106}.

This leads to the central question of the paper:
\begin{center}
\fbox{\begin{varwidth}{\columnwidth}
\centering 
Can the optimality assumption on expert demonstrations be weakened or even tolerate arbitrary outliers under offline imitation learning settings?
\end{varwidth}}
\end{center}



\begin{figure}
    \centering
    \includegraphics[width=.5\linewidth]{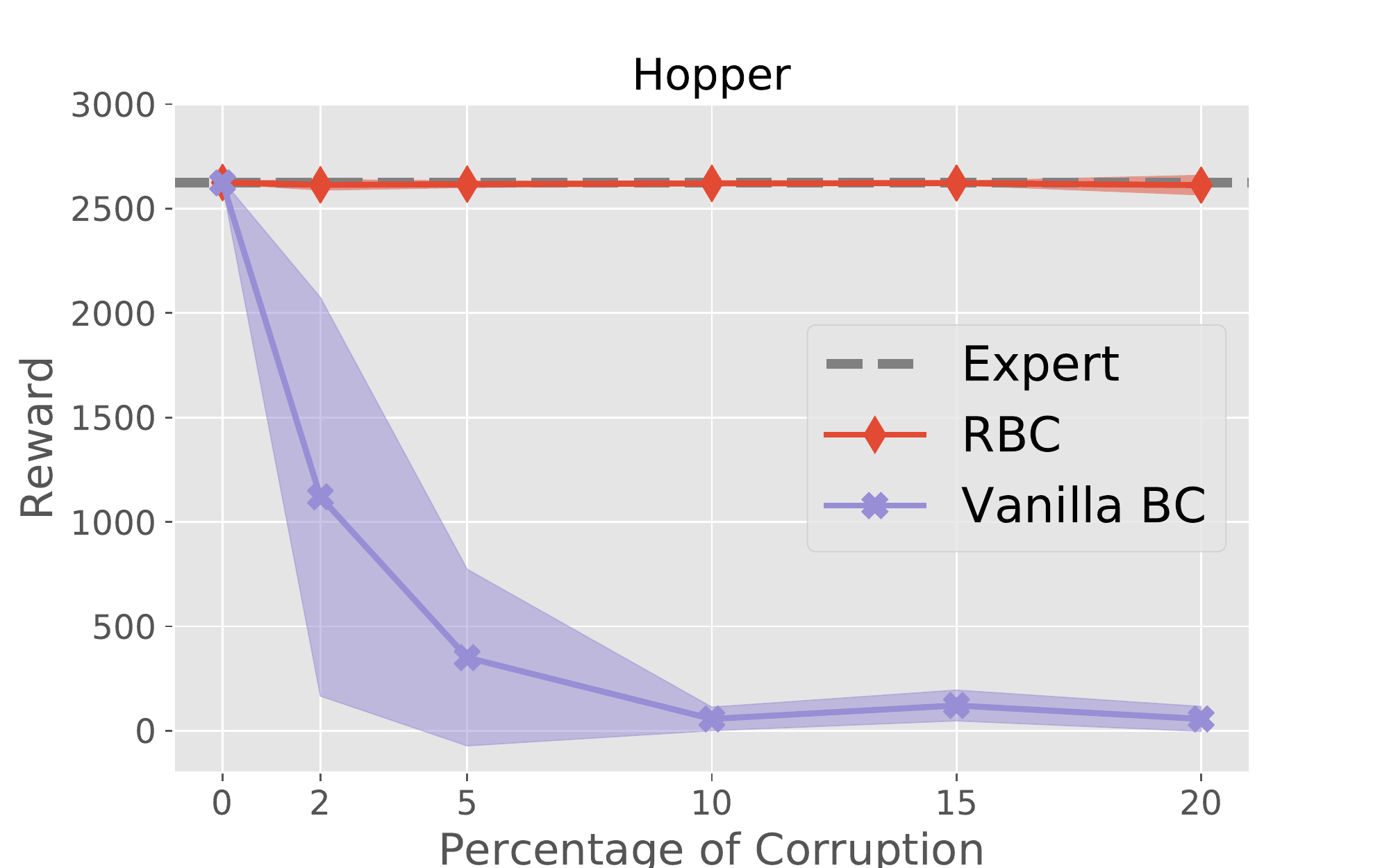}
 \caption{Reward vs. percentage of corruptions in Hopper environment
 from the PyBullet\cite{coumans2016pybullet} with corrupted demonstrations.
 We fix the sample size for the demonstration data set, and vary the fraction of corruptions $\epsilon$ up to 20\%. 
 Shaded region represents one standard deviation for 20 trials.  Our algorithm Robust Behavior Cloning (RBC) on corrupted demonstrations has nearly the same performance as BC on expert demonstrations (this is the case when $\epsilon=0$), which achieves expert level. 
 And it barely changes when $\epsilon$ grows larger to 20\%. 
 By contrast, the performance of vanilla BC on corrupted demos fails drastically. The detailed experimental setup and comparisons with existing methods are included in \Cref{fig:curve} and \Cref{fig:curve_size}.
}
  \label{fig:Curve_Hopper}
\end{figure}

More concretely,
we consider \emph{corrupted demonstrations} setting where the majority of the demonstration data is collected by an expert policy (presumably optimal), and the remaining data can be even \emph{arbitrary} outliers (the formal definition is presented in  \Cref{def:Huber}).

Such definitions allowing \emph{arbitrary} outliers for the corrupted samples have rich history in robust statistics \cite{Huber1964RobustEO, huber2011textbook}, yet 
have not been widely used in imitation learning. 
This has great significance in many applications, such as automated medical diagnosis for healthcare (\cite{RLHealthCare}) and autonomous driving \citep{ma2018improved}, where the historical data (demonstration) is often complicated and noisy which requires robustness consideration.

However, the classical \emph{offline} imitation learning approaches such as Behavior Cloning (BC) fails drastically under this corrupted demonstration settings.
We illustrated this phenomenon in \Cref{fig:Curve_Hopper}. We use BC on Hopper environment (a continuous control environment from PyBullet \cite{coumans2016pybullet}), and  the performance of the policy learned by BC drops drastically as the fraction of corruptions increases in the offline demonstration data set. 

In this paper, we propose a novel robust imitation learning algorithm -- Robust Behavior Cloning (RBC, \Cref{alg:RBC}), which is resilient to corruptions in the offline demonstrations.
Particularly, our RBC does not require potentially costly or risky interaction with the real world environment or any human annotations.
In \Cref{fig:Curve_Hopper},
Our RBC on corrupted demonstrations has nearly the same performance as BC on expert demonstrations (this is the case when $\epsilon=0$), which achieves expert level. 
 And it barely changes when $\epsilon$ grows larger to 20\%.
 The detailed experimental setup and comparisons with existing methods (e.g., \cite{NoisyBC2020}) are included in \Cref{sec:Experiments}.

\subsection{Main Contributions}
 \begin{itemize}
     \item (Algorithm) We consider robustness in  offline imitation learning where we have corrupted demonstrations.  Our definition for corrupted demonstrations   significantly weakens the 
     presumably optimal assumption on demonstration data, and can tolerate a constant $\epsilon$-fraction of state-action pairs to be arbitrarily corrupted. We refer to \Cref{def:Huber} for a more precise statement.
     
     To deal with this issue, we propose a novel algorithm Robust Behavior Cloning (\Cref{alg:RBC}) for robust imitation learning.      Our algorithm works in the offline setting, without any further interaction with the environment or any
human annotations. The core ingredient of our robust algorithm is using a novel median of means objective in policy estimation compared to classical Behavior Cloning. Hence, it's simple to implement, and computationally efficient.

     \item (Theoretical guarantees)      We analyze our Robust Behavior Cloning algorithm when there exists a constant fraction of outliers in the demonstrations under the offline setting.
     To the best of our knowledge, we provide the \emph{first
     theoretical guarantee} robust to constant fraction of arbitrary outliers in offline imitation learning.
     We show that our RBC achieves nearly the same error scaling and sample complexity compared to vanilla BC with expert demonstrations. 
     To this end, our algorithm guarantees robustness to corrupted demonstrations at no cost of statistical estimation error.
     This is the content of \Cref{sec:theory}.

     \item (Empirical support)
     We validate the predicted robustness and show the effectiveness of our algorithm on 
     a number of different 
     high-dimensional continuous control benchmarks.  
     The vanilla BC is fragile indeed with corrupted demonstrations, yet our Robust Behavior Cloning is computationally efficient, and  achieves nearly the same performance compared to vanilla BC with expert demonstrations. 
     \Cref{sec:Experiments} also shows that our algorithm achieves competitive results compared to  existing imitation learning methods. 
     
 \end{itemize}


\paragraph{Notation.}
Throughout this paper, we use $\{c_i\}_{i=1,2,3}$ to denote the universal positive constant.
We utilize the big-$O$ notation $f(n) = O(g(n))$ to denote that
there exists a positive constant $c_1$ and a natural number $n_0$ such that,
for all $n\geq n_0$, we have
$f(n) \leq c_1 g(n)$.

\paragraph{Outline.}
The rest of this paper is organized as follows. In \Cref{sec:setup}, we formally define the setup and the corrupted demonstrations.
In \Cref{sec:algo}, we introduce our RBC and the
computationally efficient algorithm (\Cref{alg:RBC}). We provide the theoretical analysis in \Cref{sec:theory}, and experimental results in \Cref{sec:Experiments}. 
We leave the detailed discussion and related works in \Cref{sec:related}.
All proofs and experimental details  are collected in the Appendix.

\section{Problem Setup}
\label{sec:setup}

\subsection{Reinforcement Learning and Imitation Learning}
\paragraph{Markov Decision Process and Reinforcement Learning.}
We start the problem setup by introducing the 
Markov decision process (MDP).
An MDP
$M = \langle \Sset, \Aset , r, \Tran, \mu_0, \gamma\rangle$
consists of a state space $\Sset$, an action space $\Aset$, an unknown reward function $r:\Sset\times \Aset \to [0, \mathsf{R}_{\max}]$, 
an unknown transition kernel $\Tran:\Sset\times \Aset \to \Delta(\Sset)$, 
an initial state distribution $\mu_0\in \Delta(\Sset)$, and a discounted factor $\gamma \in (0, 1)$.  We use $\Delta$ to denote the probability distributions on the simplex.

An agent acts in a MDP following a policy $\pi(\cdot |\obs)$, 
which prescribes a distribution over the action space $\Aset$ given each state $\obs \in \Sset$.
Running the policy starting from the initial distribution $\obs_1\sim \mu_0$ 
yields a stochastic trajectory $\mathcal{T} := \{\obs_t,\acts_t,r_t\}_{1\leq t\leq \infty}$,
where $\obs_t, \acts_t, r_t$ represent the state, action, reward at time $t$ respectively,
with $\acts_t\sim \pi(\cdot|\obs_t)$ and the next state $\obs_{t+1}$ follows the unknown transition kernel $\obs_{t+1}\sim \Tran(\cdot|\obs_t, \acts_t)$.
We denote $\rho_{\pi,t}\in \Delta(\Sset\times \Aset) $ as the marginal
joint stationary distribution for state, action at time step $t$, and we define $\rho_{\pi} = (1-\gamma) \sum_{i=1}^\infty \gamma^t \rho_{\pi,t}$ as visitation distribution for policy $\pi$.
For simplicity, we reuse the notation $\rho_\pi(s) = \int_{a\in \Aset} \rho_\pi(s,a) da$ to denote the marginal distribution over state.

The goal of reinforcement learning (RL) is to find the best policy $\pi$ to maximize the expected cumulative return $J_\pi = \E_{\mathcal{T} \sim \pi}\left[\sum_{i=1}^\infty \gamma^t r_t\right]$. 
Common RL algorithms (e.g., please refer to  \cite{csaba2010algorithms}) requires online interaction and exploration with the environments. However, this is prohibited in the offline setting.

\paragraph{Imitation Learning.}
Imitation learning (IL) aims to obtain a policy to mimic expert's behavior with demonstration data set $\Demo = \{(\obs_i, \acts_i)\}_{i=1}^N$ where $N$ is the sample size of $\Demo$. Note that we do not need any reward signal.
Tradition imitation learning assumes perfect (or near-optimal) expert demonstration -- for simplification we assume that each state-action pair $(\obs_i, \acts_i)$ is drawn from the joint stationary distribution of an expert policy $\pi_{E}$:
\begin{align}\label{equ:IL}
    (\obs_i, \acts_i) \sim \rhoExp
\end{align}

Learning from demonstrations with or without online interactions has a long history (e.g., \cite{Pomerleau1988NIPS_BC, ho2016GAIL}).
The goal of \emph{offline} IL is to learn a policy $\pih^{\rm IL} = \mathbb{A}(\Demo)$ through an IL algorithm $\mathbb{A}$, given  the demonstration data set $\Demo$, without further interaction with the unknown true transition dynamic $\Tran$.




\paragraph{Behavior Cloning.}
The Behavior Cloning (BC) is the well known algorithm  \citep{Pomerleau1988NIPS_BC} for IL which only uses offline demonstration data without any interaction with the environment.  More specifically, 
BC
solves the following Maximum Likelihood Estimation (MLE) problem, which minimizes the average Negative Log-Likelihood (NLL) for all samples in offline demonstrations $\Demo$:
\begin{align}\label{equ:pi_BC}
    \pih^{\rm BC} = \arg \min_{\pi \in \Pi} \frac{1}{ N}  \sum_{(s, a) \in 
    \Demo } -\log(\pi(a | s)).
\end{align} 

Recent works \citep{Agarwal2019AJKS,Jiao2020FundamentalLimits, YuYang2021error, xu2021TAIL} have shown that BC is optimal under the offline setting, and can only be improved with the knowledge of transition dynamic $\Tran$ \emph{in the worst case}. 
Also, another line of research considers improving BC with further online interaction of the environment \citep{brantley_SunWen_2019disagreement} or actively querying an expert \citep{Ross2011AIStats, ross2014reinforcement}.

\subsection{Learning from corrupted  demonstrations
}

 However, it is sometimes unrealistic to assume that the demonstration data set is collected through a presumably optimal expert policy. In this paper, we propose \Cref{def:Huber} for the corrupted demonstrations, which tolerates gross corruption or model mismatch in offline data set.

\begin{definition}[Corrupted Demonstrations]\label{def:Huber}
Let the state-action pair $(\obs_{i}, \acts_{i})_{i=1}^N$  drawn from the joint stationary distribution of a presumably optimal  expert policy $\pi_{E}$.
The corrupted demonstration data
$\Demo$ are generated by the following process: an adversary can choose an
arbitrary $\epsilon$-fraction ($\epsilon < 0.5$) of the samples in $[N]$ and modifies them with arbitrary values.
We note that $\epsilon$ is a constant independent of the dimensions of the problem.
After the corruption, we use $\Demo$ to denote the corrupted demonstration data set.
\end{definition}

This corruption process can represent gross corruptions or model mismatch in the demonstration data set.
To the best of our knowledge, \Cref{def:Huber} is the first definition for corrupted demonstrations in imitation learning which tolerates \emph{arbitrary} corruptions.

In the supervised learning, the well-known Huber's contamination model (\cite{Huber1964RobustEO, huber2011textbook}) considers $(\bm{x}, y) \overset{iid}{\sim} (1-\epsilon)P + \epsilon Z,$ 
where $\bm{x} \in \Real^d$ is the explanatory variable (feature)  and  $y\in\Real$ is the response variable.
Here, $P$ denotes the \emph{authentic} statistical distribution such as Normal mean estimation or linear regression model, and $Z$ denotes the outliers. 

Dealing with corrupted $\bm{x}$ and $y$ in high dimensions has a long history in the robust statistics community \citep[e.g.][]{rousseeuw1984least,chen2013robust, chen2017distributed, Yin_median}.
However, it's only until recently that robust statistical methods can handle \emph{constant} $\epsilon$-fraction (independent of dimensionality $\Real^d$) of outliers in $\bm{x}$ and $y$
\citep{klivans2018efficient,ravikumar2020RGD, sever2018, liu2019, liuAISTATS,shen2019ITLM,lugosi2019risk,lecue2020robust, jalal2020robust}.
We note that in Imitation Learning, the data collecting process for the demonstrations does not obey i.i.d. assumption in traditional supervised learning due to the temporal dependency.

\section{Our Algorithm}
\label{sec:algo}


\Cref{equ:pi_BC} directly 
minimizes the empirical mean of Negative Log-Likelihood, and it is widely known that the mean operator is fragile to corruptions \cite{Huber1964RobustEO,huber2011textbook}.
Indeed, our
experiment in \Cref{fig:Curve_Hopper} demonstrates that in the presence of outliers, vanilla BC 
fails drastically.
Hence, we consider using a robust estimator to replace the empirical average of NLL in \cref{equ:pi_BC} -- we first introduce the classical Median-of-Means (MOM) estimator for the mean estimation,
and then adapt it to dealing with loss functions in 
robust imitation learning problems.

The vanilla MOM estimator for one-dimensional mean estimation  works like following: 
(1) randomly partition $N$ samples into $M$ batches;
(2) calculates the mean for each batch;
(3) outputs the median of  these batch mean.

The MOM mean estimator achieves sub-Gaussian concentration bound for one-dimensional mean estimation even though the underlying  distribution only has second moment bound (heavy tailed distribution)
(interested readers are referred to textbooks such as \cite{MOM_nemirovsky1983problem, MOM_jerrum1986random, MOM_alon1999space}).
Very recently, MOM estimators are used 
for high dimensional robust regression \citep{brownlees2015empirical, hsu2016loss, lugosi2019risk, lecue2020robust, jalal2020robust} by applying MOM estimator on the \emph{loss function} of empirical risk minimization process.

\subsection{
Robust Behavior Cloning}

Inspired by the MOM estimator,
a natural robust version  of  \cref{equ:pi_BC} 
can randomly partition $N$
samples into $M$ batches
with the batch size $b$, and calculate
\begin{align}\label{equ:MOM_minimization}
\pih =  \arg \min_{\pi \in \Pi} \underset{1\leq j \leq M}{\mathrm{median}} 
\left( \ell_j(\pi) \right),
\end{align}
where the loss function $\ell_j(\pi)$ is the average Negative Log-Likelihood in the batch $B_j, j \in [M]$:
\begin{align} \label{equ:loss}
    \ell_j(\pi) = \frac{1}{b}\sum_{(s, a) \in B_j}-\log(\pi(a | s)).
\end{align}
Our idea \cref{equ:MOM_minimization} minimizes the MOM of NLL, which extends the MOM mean estimator to the loss function for robust imitation learning.
Although \cref{equ:MOM_minimization} can also achieve robust empirically result, we propose 
\Cref{def:RBC}
for theoretical convenience, which
optimizes the
min-max version (MOM tournament \cite{LeCam2012asymptotic, lugosi2019risk, lecue2020robust, jalal2020robust})
to handle arbitrary outliers in demonstration data set $(\obs, \acts) \in \Demo$. 

\begin{defi}[Robust Behavior Cloning]\label{def:RBC}
We split the corrupted demonstrations $\Demo$ into $M$ batches randomly\footnote{Without loss of generality, we assume that $M$ exactly divides the sample size $N$, and $b = \frac{N}{M}$ is the batch size.}: $\{B_j\}_{j=1}^{M}$, with the batch size $b \leq  \frac{1}{3 \epsilon}$.
The Robust Behavior Cloning solves the following optimization
\begin{align}\label{equ:pi_MOMBC}
    \pir =  \arg \min_{\pi \in \Pi} \max_{\pi' \in \Pi}   \underset{1\leq j \leq M}{\mathrm{median}} 
    \left( \ell_j(\pi) - \ell_j(\pi') \right).
\end{align} 
\end{defi}

The workhorse of \Cref{def:RBC} is \cref{equ:pi_MOMBC}, which uses a novel variant of MOM tournament procedure 
for imitation learning problems.

In \cref{equ:loss}, we calculate the average Negative Log-Likelihood (NLL) for a single batch of state-action pair $(s, a)$, and  
$\pir$ is the solution of a min-max formulation based on the batch loss $\ell_j(\pi)$ and $\ell_j(\pi')$. Though our algorithm minimizes the robust version of NLL, we do not utilize the traditional iid assumption in the supervised learning.

The \emph{key results} in our theoretical analysis 
show that the min-max solution to the  
median batch
of the loss function
is robust to a constant fraction of arbitrary outliers in the demonstrations.
The intuition behind solving this min-max formulation is that the inner variable $\pi'$ needs to get close to $\pis$ to maximize the term $\underset{1\leq j \leq M}{\mathrm{median}} 
    \left( \ell_j(\pi) - \ell_j(\pi') \right)$;
    and the outer variable
$\pi$ also need to get close to $\pis$ to minimize the term. 
In \Cref{sec:theory}, we show that under corrupted demonstrations,
$\pir$ will be close to $\pis$.
In particular, $\pir$ in \Cref{def:RBC} has the same error scaling and sample complexity compared to $\pih^{\rm BC}$ in the expert demonstrations setting.

\textbf{Algorithm design.}
In \Cref{sec:theory}, we provide rigorous statistical  guarantees for  \Cref{def:RBC}.
However,  the objective function \cref{equ:pi_MOMBC} is not convex (in general),
hence  we use \Cref{alg:RBC} as a computational heuristic to solve it.  

In each iteration of  \Cref{alg:RBC}, we randomly partition the demonstration data set $\Demo$ into $M$ batches, and calculate the loss  $\ell_j(\pi) - \ell_j(\pi')$  by \cref{equ:loss}. We then pick the batch $B_{\rm Med}$ with the median loss, and evaluate the gradient on that batch. We use gradient descent on $\pi$ for the $\arg \min$ part and gradient ascent on $\pi'$ for the  $\arg \max$ part.

In \Cref{sec:Experiments}, 
we empirically show that this gradient-based
heuristic \Cref{alg:RBC} is able to minimize this objective and has good convergence properties. As for the time complexity, when using back-propagation on one batch of samples, our RBC incurs  overhead costs compared to vanilla BC, in order to evaluate the loss function for all samples via forward propagation. Empirical studies in \Cref{sec:Exp_appendix} show that the time complexity of RBC is comparable to vanilla BC. 


\begin{algorithm}[t]
\begin{algorithmic}[1]
\STATE \textbf{Input:} 
{Corrupted demonstrations  $\Demo$}
\STATE \textbf{Output:} Robust policy $\pir$ \\
{\kern2pt \hrule \kern2pt}
\STATE Randomly initialize $\pi$ and $\pi'$ respectively.
\FOR {$t = 0$ to $T-1$,}
\STATE Randomly partition $\Demo$ to $M$ batches with the batch size $b \leq  \frac{1}{3 \epsilon}$. 
\STATE For each batch $j\in [M]$, calculate the loss  $\ell_j(\pi) - \ell_j(\pi')$  by \cref{equ:loss}.
\STATE Pick the batch with median loss within $M$ batches
$$ \underset{1\leq j \leq M}{\mathrm{median}} 
    \left( \ell_j(\pi) - \ell_j(\pi') \right),$$
and evaluate the gradient for $\pi$ and $\pi'$
using back-propagation on that batch\\
(i) perform gradient descent on $\pi$. \\
(ii) perform gradient ascent on $\pi'$.
\ENDFOR
\STATE \textbf{Return:} Robust policy $\pir = \pi$.
\end{algorithmic}
\caption{Robust Behavior Cloning.}
\label{alg:RBC}
\end{algorithm}

\section{Theoretical Analysis}
\label{sec:theory}

In this section, we provide theoretical guarantees for our RBC algorithm. Since our method (\Cref{def:RBC}) directly estimates the conditional probability $\pi(a | s)$ over the offline demonstrations, our theoretical analysis provides guarantees on $\Expe_{s \sim \dist} \normtv{ \pir(\cdot|\obs) - \pis(\cdot|\obs)}^2$, which upper bounds the total variation norm compared to $\pis$ under the expectation of $s \sim \dist$. 
The ultimate goal of the learned policy is to maximize the expected cumulative return, thus we then provide an upper bound for the sub-optimality $J_{\pis} - J_{\pir}$.

We begin the theoretical analysis by \Cref{ass:discrete}, which simplifies our analysis and is common in literature \citep{Agarwal2019AJKS, Agarwal2020MLE}. By assuming that the policy class $\Pi$ is discrete, our upper bounds depend on the quantity ${\log(|\Pi|)} / {N}$, which matches the error rates and sample complexity for using BC with expert demonstrations \citep{Agarwal2019AJKS, Agarwal2020MLE}.

\begin{assm}\label{ass:discrete}
We assume that the policy class
$\Pi$ is discrete, and realizable, i.e., $\pis \in \Pi$.
\end{assm}

\subsection{The upper bound for the policy distance}
We first present \Cref{thm:pi_bound}, which shows that minimizing the MOM objective via \cref{equ:pi_MOMBC} guarantees the closeness of robust policy to optimal policy in total variation distance.

\begin{theo}\label{thm:pi_bound}\
Suppose we have corrupted demonstration data set  $\Demo$ with sample size $N$ from 
\Cref{def:Huber}, and there exists a constant corruption ratio $\epsilon < 0.5$.
Under \Cref{ass:discrete}, 
let $\tau$ to be the output objective value with $\pir$ in the optimization \cref{equ:pi_MOMBC} with the batch size $b \leq  \frac{1}{3 \epsilon}$, then 
with probability at least $1-c_1\delta$,
we have
\begin{align}\label{equ:pi_bound}
    \Expe_{s \sim \dist} \normtv{ \pir - \pis}^2 = O \left( \frac{\log(|\Pi|/\delta)}{N} + \tau \right).
\end{align}
\end{theo}

The proof is collected in \Cref{sec:proof}. We note that the data collection process does not follow the iid assumption, hence we use martingale analysis similar to 
 \citep{Agarwal2019AJKS, Agarwal2020MLE}.
The first part of \cref{equ:pi_bound} is the statistical error $\frac{\log(|\Pi|/\delta)}{N}$, 
which matches the error rates of vanilla BC for expert demonstrations \citep{Agarwal2019AJKS, Agarwal2020MLE}. 
The second part is the final objective value in the optimization \cref{equ:pi_MOMBC} $\tau$ which includes two parts -- the first part scales with $O(\frac{1}{b})$, which is equivalent to the fraction of corruption 
$O(\epsilon)$. 
The second part is the
sub-optimality gap 
due to the solving the non-convex optimization.
Our main theorem -- \Cref{thm:pi_bound} -- guarantees that a small value of the final objective implies an accurate estimation of policy and hence we can certify estimation quality using the obtained final value of the objective.

\subsection{The upper bound for the sub-optimality }
Next, we present \Cref{thm:V_bound}, which guarantees the reward performance of the learned robust policy $\pir$. 

\begin{theo}\label{thm:V_bound}
Under the same setting as \Cref{thm:pi_bound}, we have
\begin{align}\label{equ:V_bound}
    J_{\pis} - J_{\pir} \leq 
  O \left( \frac{1}{(1 - \gamma)^2} 
   \sqrt{\frac{\log(|\Pi|/\delta)}{N} + \tau } \right),
\end{align}
with probability at least $1-c_1\delta$.
\end{theo}

The proof is collected in \Cref{sec:proof}.  
We note that the error scaling and sample complexity of the statistical error ${\log(|\Pi|/\delta)}/{N}$ 
in \Cref{thm:V_bound} match the vanilla BC with expert demonstrations \citep{Agarwal2019AJKS, Agarwal2020MLE}.

\begin{remark}\label{remark:shift}
The quadratic dependency on the effective horizon ($\frac{1}{(1 - \gamma)^2}$ in the discounted setting or $H^2$ in the episodic setting) is widely known as the compounding error or distribution shift in literature, which is due to the essential limitation of offline imitation learning setting. Recent work \citep{Jiao2020FundamentalLimits, YuYang2021error} shows that this quadratic dependency cannot be improved without any further interaction with the environment or the knowledge of transition dynamic $\Tran$. Hence BC is actually optimal under no-interaction setting.
Also, a line of research considers improving BC by further 
online interaction with the environment or even active query of the experts \citep{Ross2011AIStats, brantley_SunWen_2019disagreement, ross2014reinforcement}. 
Since our work, as a robust counterpart of BC,  focuses on  the robustness to the corruptions in the offline demonstrations setting, it  can be naturally used in  the online setting such as DAGGER  \citep{Ross2011AIStats} and \cite{brantley_SunWen_2019disagreement}.
\end{remark}



\section{Experiments}
\label{sec:Experiments}






In this section, we study the empirical performance of our Robust Behavior Cloning. We evaluate the robustness of Robust Behavior Cloning on several continuous control benchmarks simulated by PyBullet \cite{coumans2016pybullet} simulator: HopperBulletEnv-v0, 
Walker2DBulletEnv-v0,
HalfCheetahBulletEnv-v0
and AntBulletEnv-v0.
Actually, these tasks have true reward function already in the simulator. We will use \emph{only} state observation and action for the imitation algorithm, and we then use the reward to evaluate the obtained policy when running in the simulator.

\subsection{Experimental setup}
For each task, we collect the presumably optimal expert trajectories using pre-trained agents from Standard Baselines3\footnote{The pre-trained agents were cloned from the following repositories: \url{https://github.com/DLR-RM/stable-baselines3}, \url{https://github.com/DLR-RM/rl-baselines3-zoo}.}. 
In the experiment, we use Soft Actor-Critic  \cite{haarnoja2018SAC} in the Standard Baselines3 pre-trained agents, and we consider it to be an expert.
We provide the  hyperparameters setting in  \Cref{sec:Exp_appendix}.

For the continuous control environments, the action space are  bounded between -1 and 1.
We note that \Cref{def:Huber} allows for \emph{arbitrary} corruptions, and we choose these outliers' action such that it has the maximum effect, and cannot be easily detected. 
We generate corrupted demonstration data set $\Demo$ as follows: we first  randomly choose $\epsilon$ fraction of samples, and corrupt the actions.
Then, for the option (1), we set the actions of outliers to the boundary ($-1$ or $+1$). 
For the  option (2),  the actions of outliers are drawn from a 
\emph{uniform distribution} between $-1$ and $+1$.

We compare our RBC algorithm (\Cref{alg:RBC}) to a number of natural baselines: the first baseline
is directly using BC on the corrupted demonstration $\Demo$ without any robustness consideration.
The second one is using BC on the \emph{expert demonstrations} (which is equivalent to $\epsilon = 0$ in our corrupted demonstrations) with the same sample size.  

We also investigate the empirical performance of the baseline which achieves the state-of-the-art performance: Behavior Cloning from Noisy Observation (Noisy BC).
Noisy BC is a recent offline imitation learning algorithm
proposed by \cite{NoisyBC2020},
which  achieves superior performance compared to \cite{wu2019imperfect,brown2019TREX, brantley_SunWen_2019disagreement}. Similar to our RBC,
Noisy BC does not require any environment interactions during training or
any screening/annotation processes to discard the non-optimal behaviors in the demonstration data set.

The Noisy BC works in an iterative fashion:
in each iteration, it reuses the old policy  iterate $\pi^{\rm old}$ to re-weight the state-action samples via
the weighted Negative Log-Likelihood
\begin{align*}
    \pi^{\rm new} = \arg \min_{\pi \in \Pi} \frac{1}{ N}  \sum_{(s, a) \in 
    \Demo } -\log(\pi(a | s)) * \pi^{\rm old}(a | s).
\end{align*} 
Intuitively, if the likelihood $\pi^{\rm old}(a | s)$ is small in previous iteration, the weight for the state-action sample $(s, a)$ will be small in the current iteration. Noisy BC outputs $\pih$ after multiple iterations.

\subsection{Convergence of our algorithm
}
We  first illustrate the convergence and the performance of our algorithm  by 
tracking the metric of different algorithms vs. epoch number in the whole training process. 
More specifically,
we evaluate current policy in the simulator for 20 trials, and obtain the mean and standard deviation of cumulative reward for every 5 epochs.
 This metric corresponds to theoretical bounds in \Cref{thm:V_bound}.

We focus on four continuous control environments, where the observation space has dimensions around 30, and 
the action space has boundary $[-1, 1]$.
In this experiment, we adopt option (1), which set the actions of outliers to the boundary ($-1$ or $+1$).
We fix the corruption ratio as 10\% and 20\%, and present the Reward vs. Epochs.
Due to the space limitation, we leave the experiments for all the environments to \Cref{fig:epochs_appendix} in \Cref{sec:Exp_appendix}.

As illustrated in \Cref{fig:epochs_appendix}, 
Vanilla BC on corrupted demonstrations fails to converge to expert policy. 
Using our robust counterpart \Cref{alg:RBC} on corrupted demonstrations has good convergence properties. Surprisingly, our RBC on corrupted demonstrations has nearly the same reward performance vs. epochs of directly using
BC on \emph{expert demonstrations}.

\textbf{Computational consideration.} 
Another important aspect of our algorithm is the computational efficiency.
To directly compare the time complexity, we report the reward vs. wall clock time performance of our RBC and ``Oracle BC'', which optimizes on the expert demonstrations.
The experiments are conducted on 1/2 core of NVIDIA T4 GPU, and we leave the results to \Cref{tab:time} in \Cref{sec:Exp_appendix} due to space limitations.
When using back-propagation on  batches of samples,
our RBC incurs overhead costs compared to vanilla BC,
in order to evaluate the loss function for all samples via
forward propagation. \Cref{tab:time} shows that
the actual running time time of RBC is comparable to vanilla BC.

\begin{figure*}[t]
\centering
\begin{subfigure}{\linewidth}
  \centering
  \includegraphics[width=.425\linewidth]{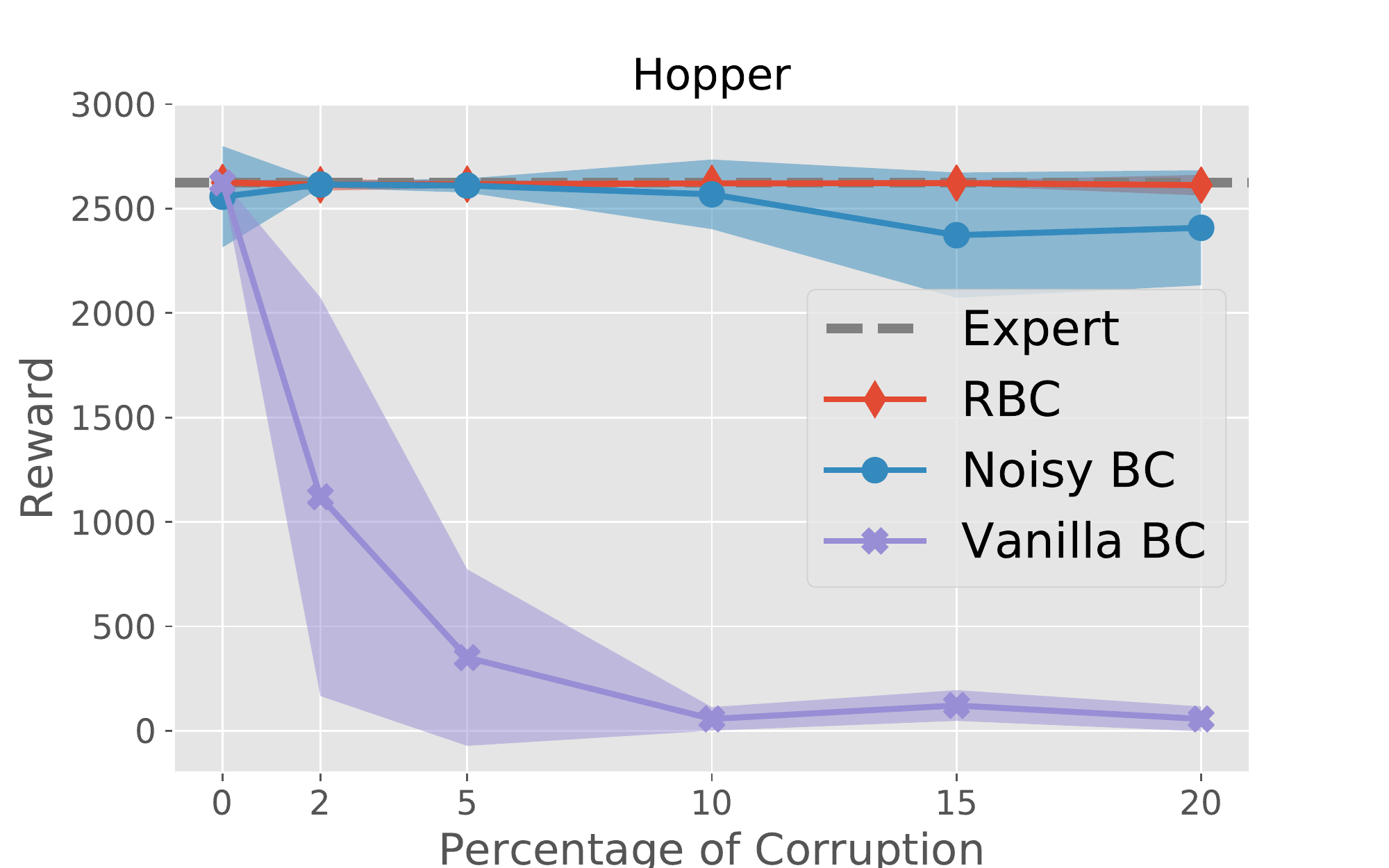}
\includegraphics[width=.425\linewidth]{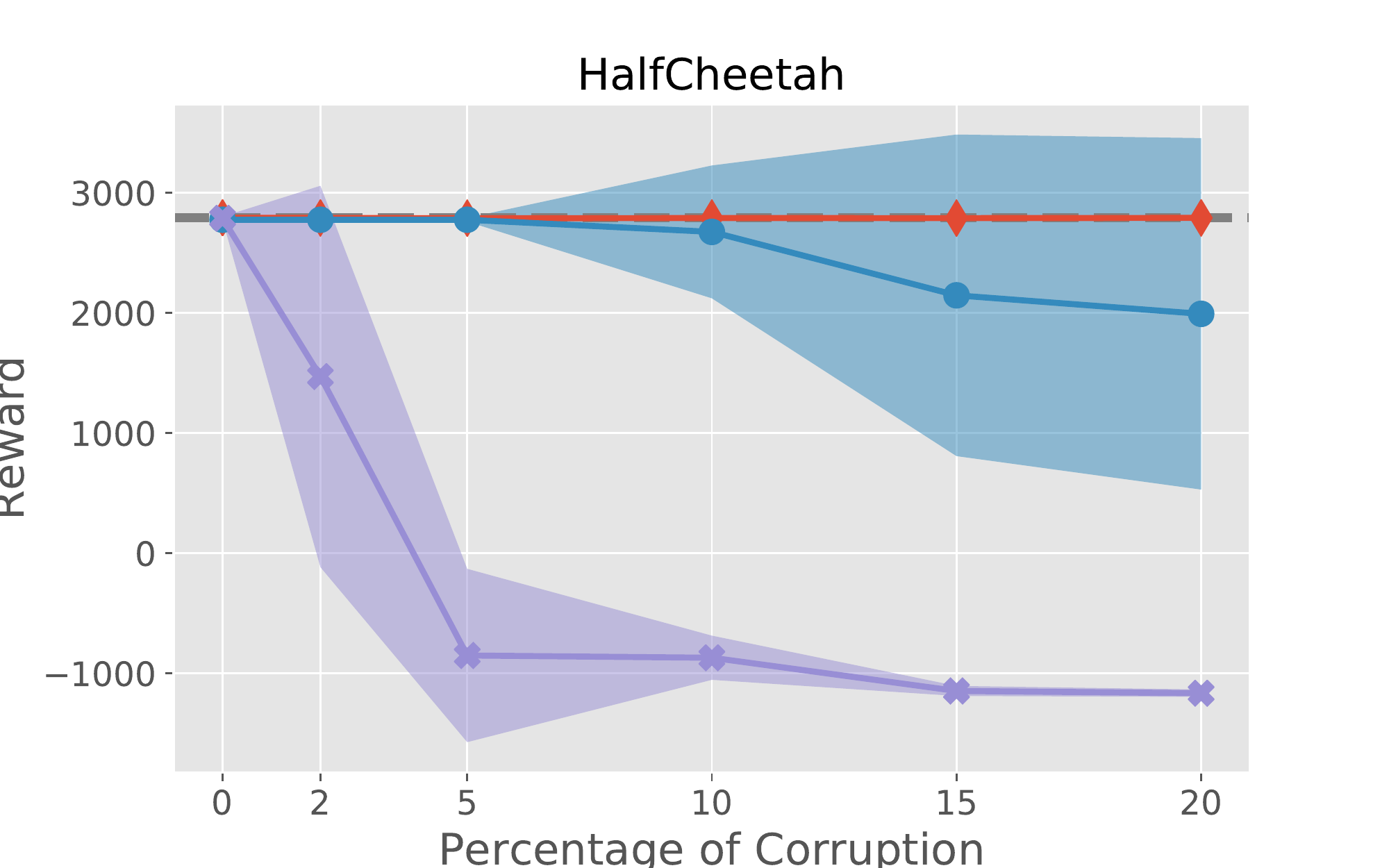}
\includegraphics[width=.425\linewidth]{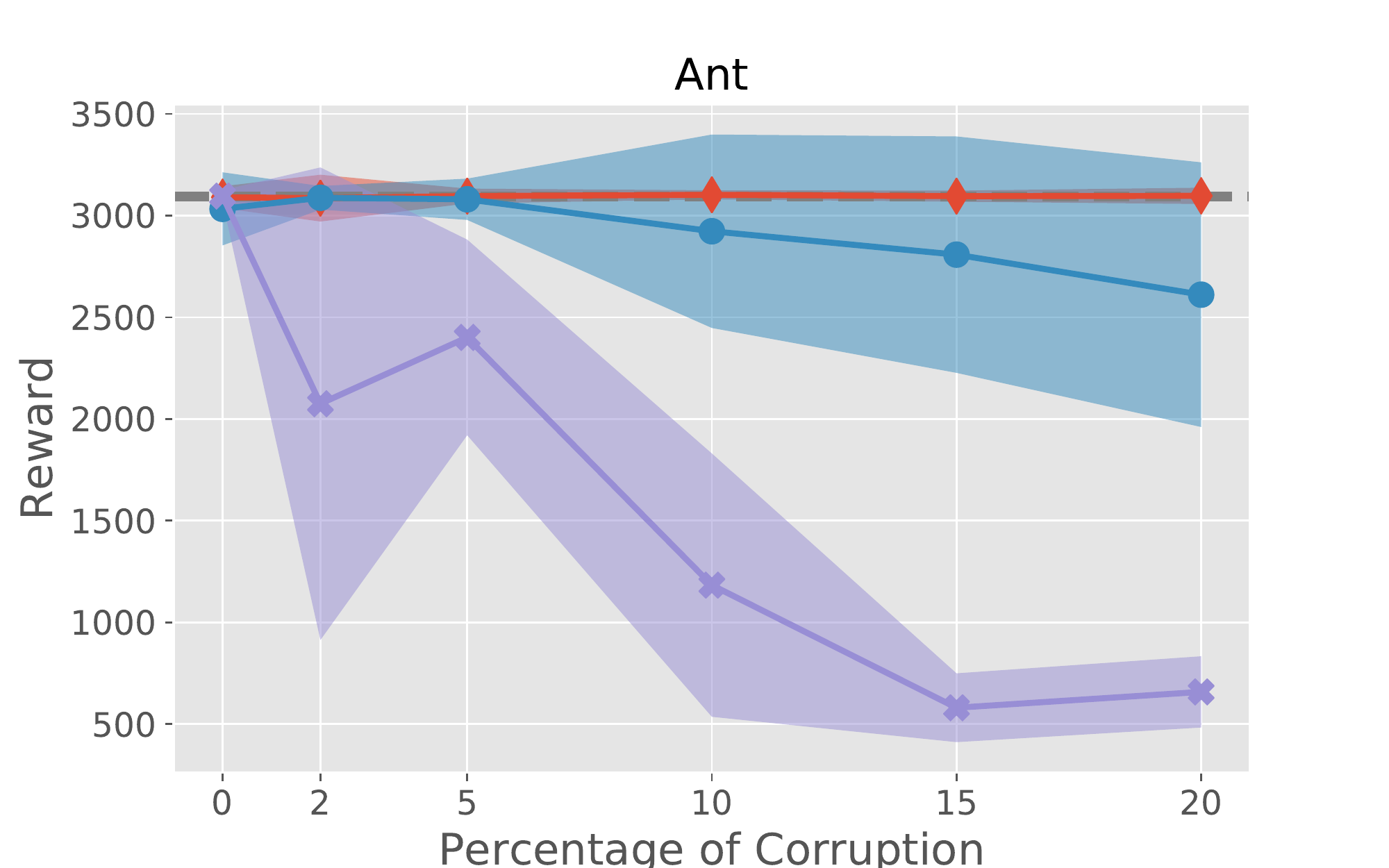}
\includegraphics[width=.425\linewidth]{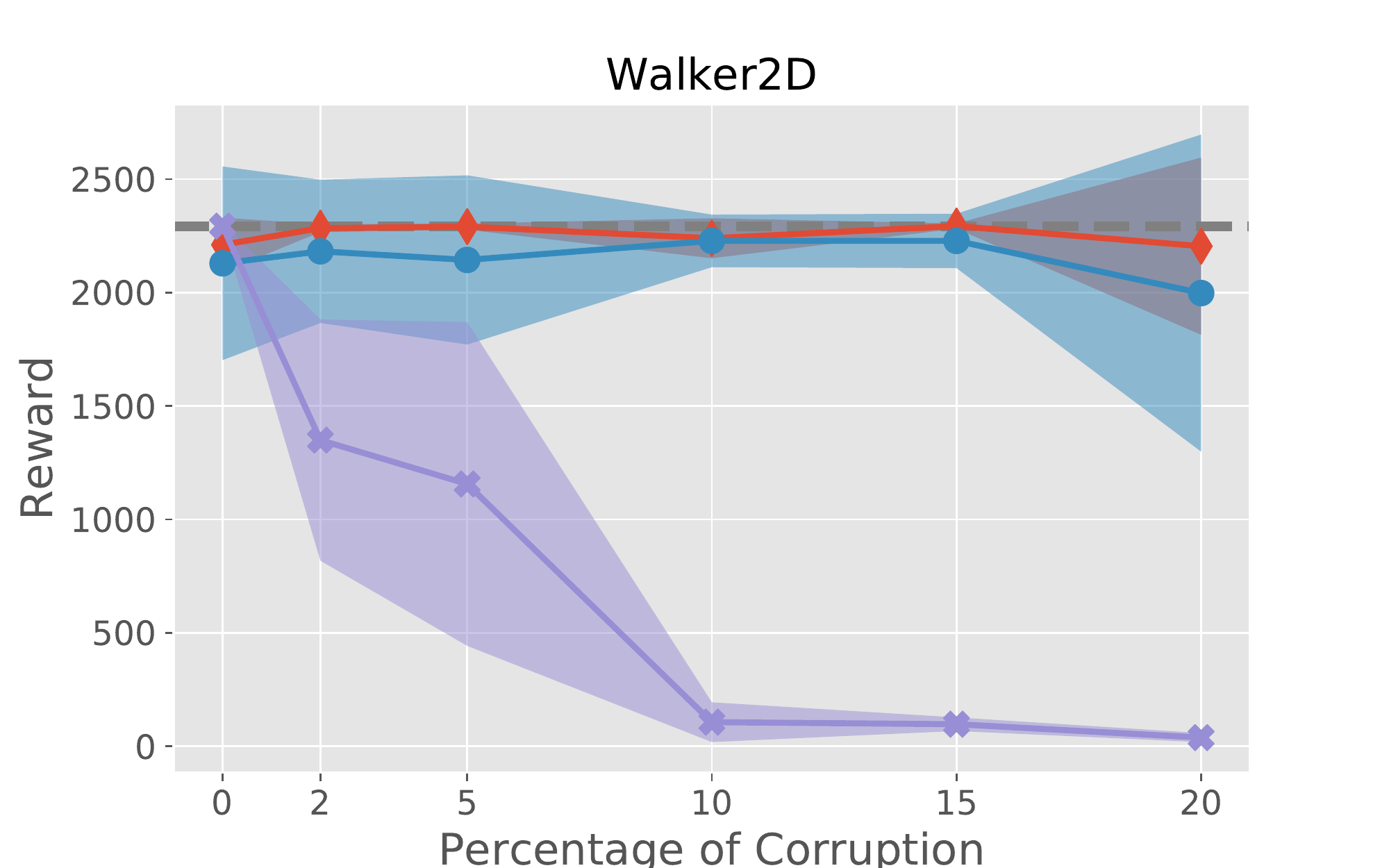}
\caption{The boundary case.}
\end{subfigure}%

\begin{subfigure}{\linewidth}
  \centering
  \includegraphics[width=.425\linewidth]{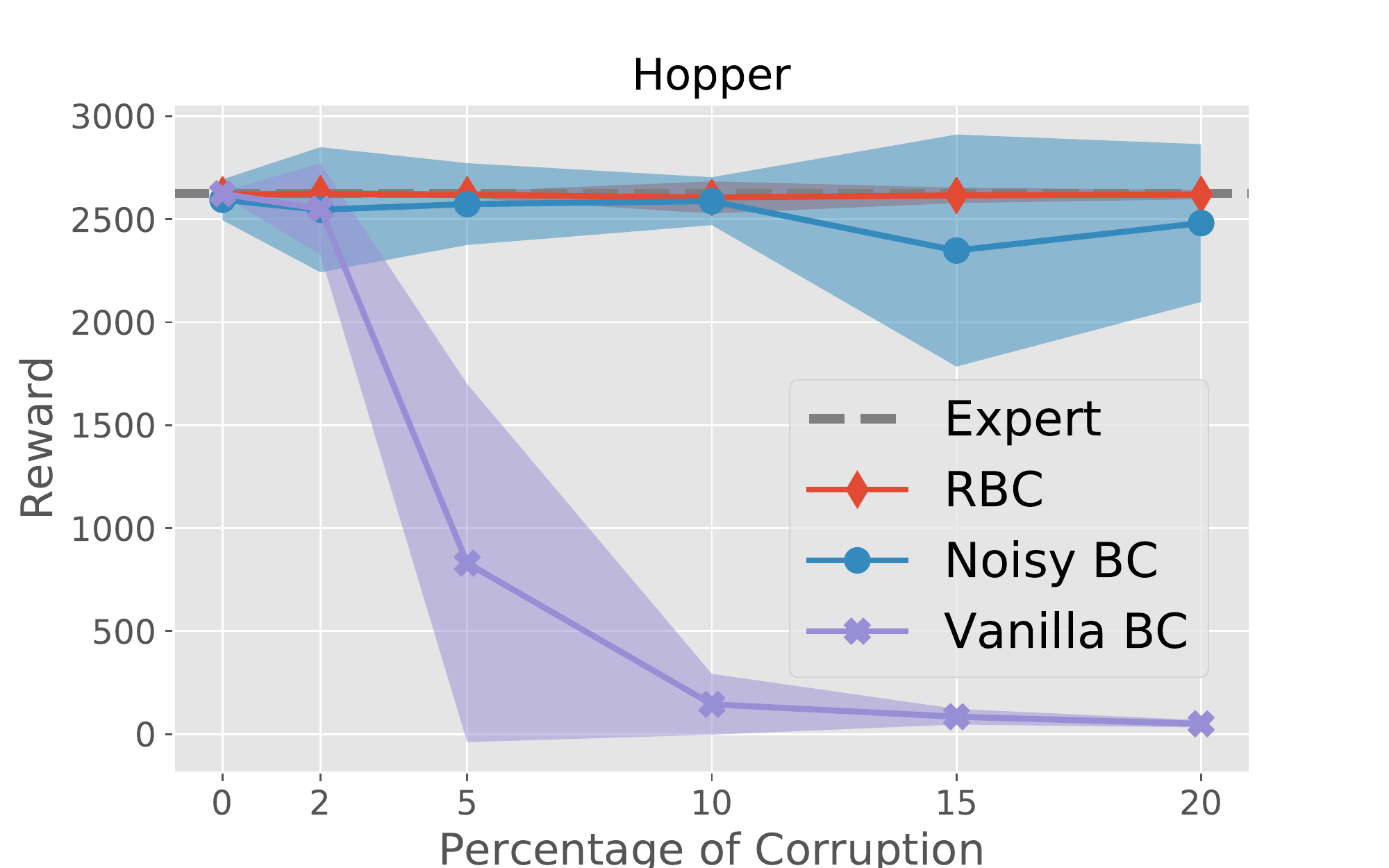}
\includegraphics[width=.425\linewidth]{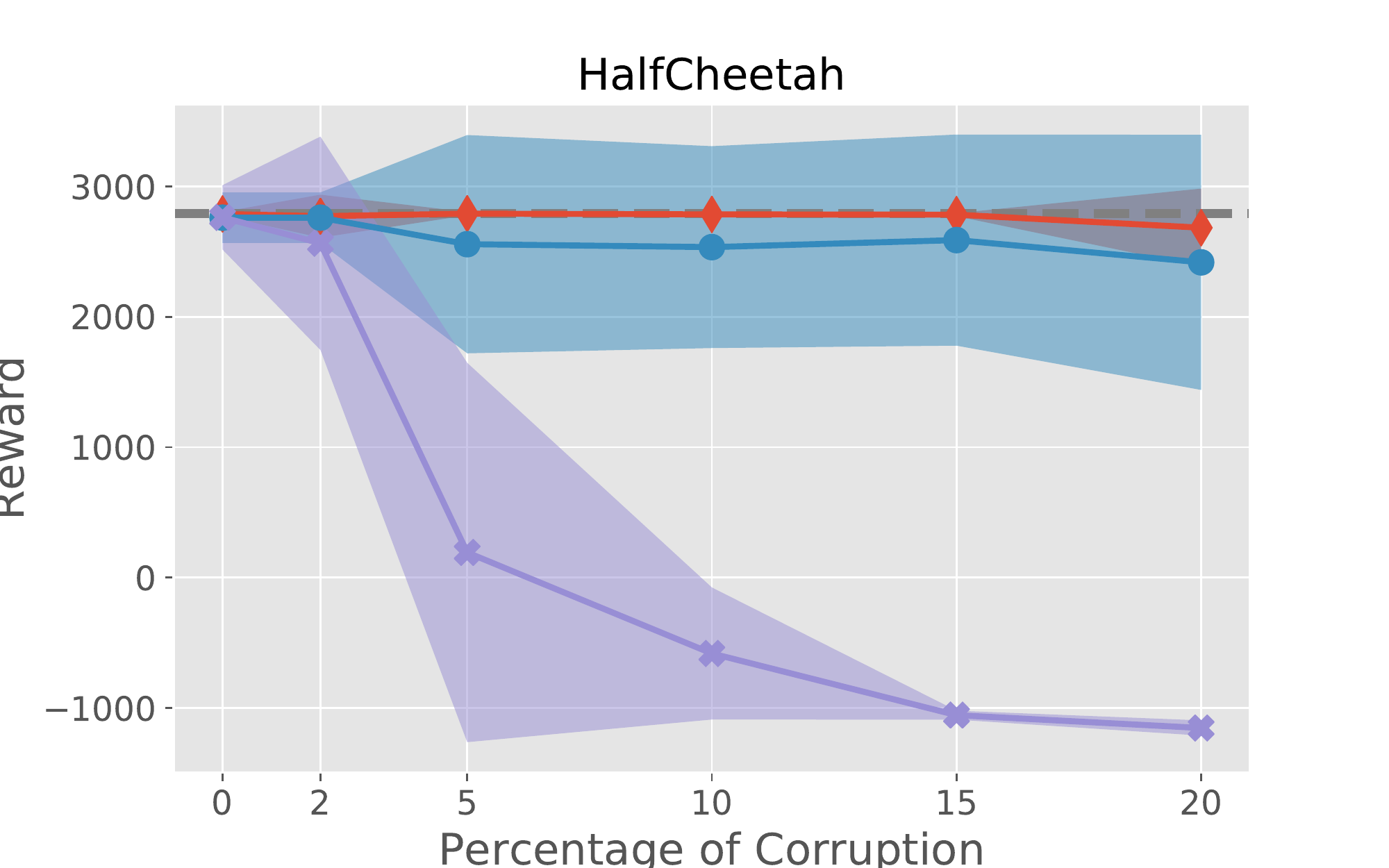}
\includegraphics[width=.425\linewidth]{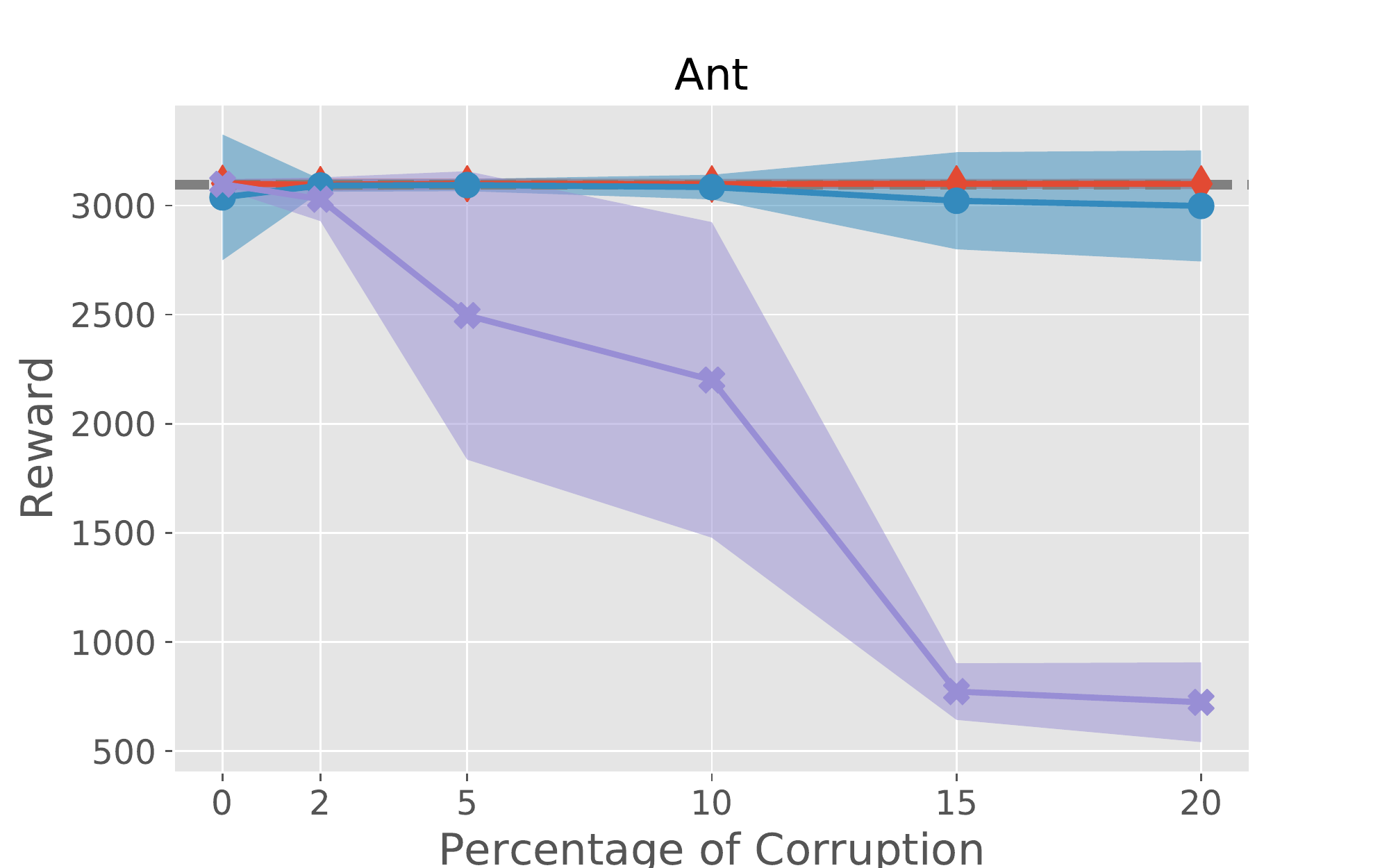}
\includegraphics[width=.425\linewidth]{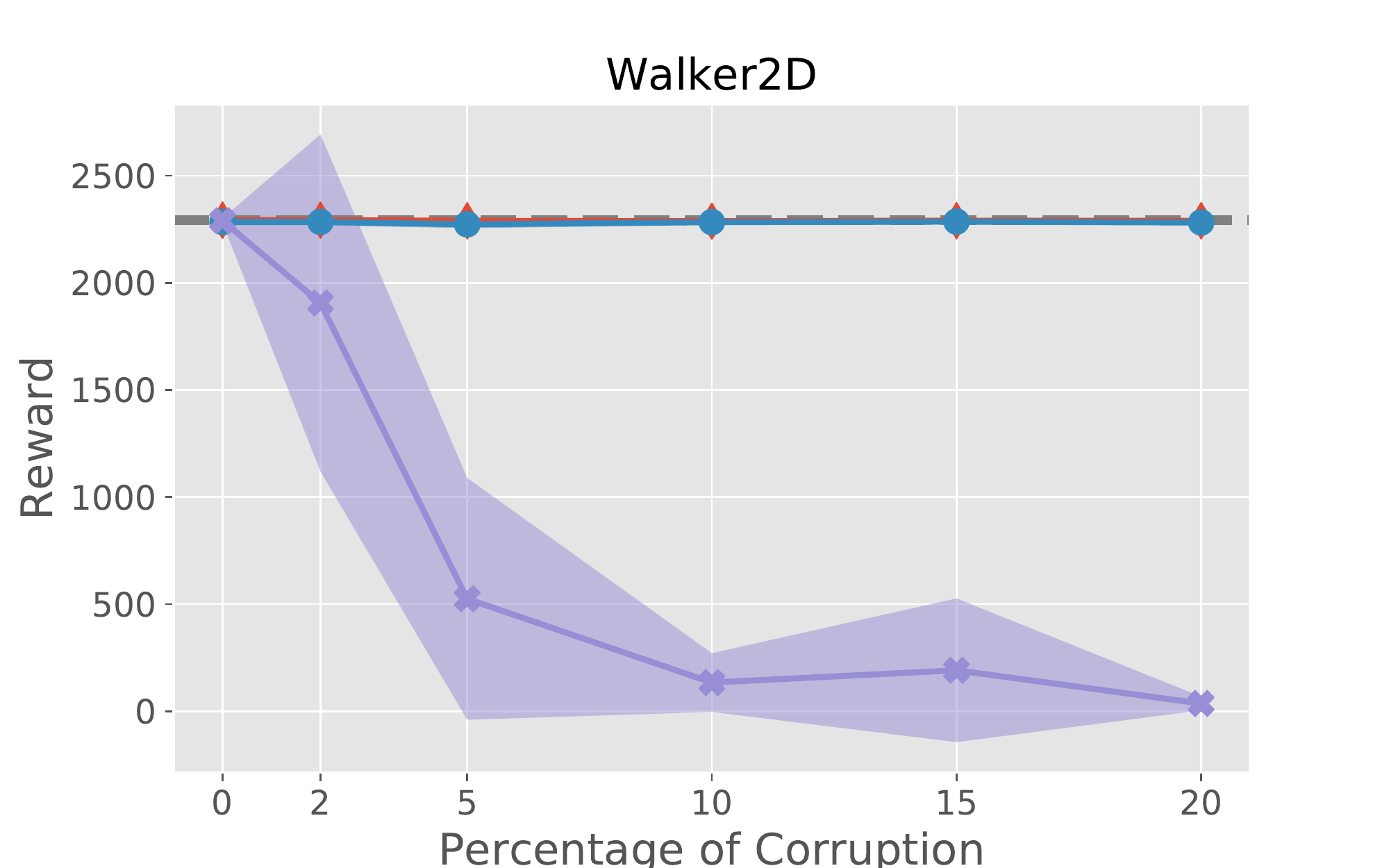}
\caption{The uniform case.}
\end{subfigure}%
\caption{Reward vs. Percentage of Corruption $\epsilon$ 
with demonstration data of size 60,000.
We vary the corruption ration  $\epsilon$ from  0\% to  20\%.
In (a), the outliers are set randomly at the boundary -1 or 1.
In (b), the outliers are randomly drawn from a uniform distribution between $[-1, 1]$.
For different algorithms, we report their reward performance after training with 200 epochs.
The shaded region represents the standard deviation for 20 trials. Vanilla BC (purple line) on corrupted
demonstrations fails drastically. 
Surprisingly, our RBC (red line) on corrupted demonstrations has nearly the same  performance as the expert (grey dashed line), 
and achieves competitive results compared to the state-of-the-art robust imitation learning method (blue line) \cite{NoisyBC2020}.}
\label{fig:curve}
\end{figure*}

\subsection{Experiments under different setups}

In this subsection, we compare the performance of our algorithm and existing methods under different setups. 

\textbf{Different fraction of corruption.} 
We fix the sample size as 60,000,  and vary
the corruption fraction $\epsilon$ from 0\% to 20\%. 
\Cref{fig:Curve_Hopper} has shown  that our RBC is resilient to outliers in the demonstrations ranging from 0 to 20\% for the Hopper environment.
In \Cref{fig:curve},
the full experiments 
validate our theory that
our Robust Behavior Cloning nearly matches the expert performance for different environments and corruption ratio.
By contrast, the performance of vanilla BC on corrupted demos fails drastically.

\textbf{Different sample size in $\Demo$.}
In this experiment, we fix the fraction of the corruption $\epsilon = 15\%$, 
set the actions of outliers to the boundary ($-1$ or $+1$),
and
vary the sample size of the demonstration data set. It is expected in \Cref{thm:V_bound} that larger sample
size of $\Demo$ leads to smaller suboptimality gap in value function. 
\Cref{fig:curve_size} validates our theory: Our
RBC on corrupted demonstrations has nearly the same reward as Oracle BC (BC on expert demonstrations), and
the sub-optimality gap gets smaller as sample size $N$ grows larger. However, directly using BC on corrupted
demonstrations cannot improve as sample size grows larger.

The Oracle BC on expert demonstrations is a strong baseline which achieves reward
performance at expert level with very few transition samples in demonstrations. This suggests that
compounding error or distribution shift may be less of a problem in these environments. This is
consistent with the findings in \cite{brantley_SunWen_2019disagreement}.

\begin{figure*}[t]
\centering
\begin{subfigure}{\linewidth}
  \centering
  \includegraphics[width=.45\linewidth]{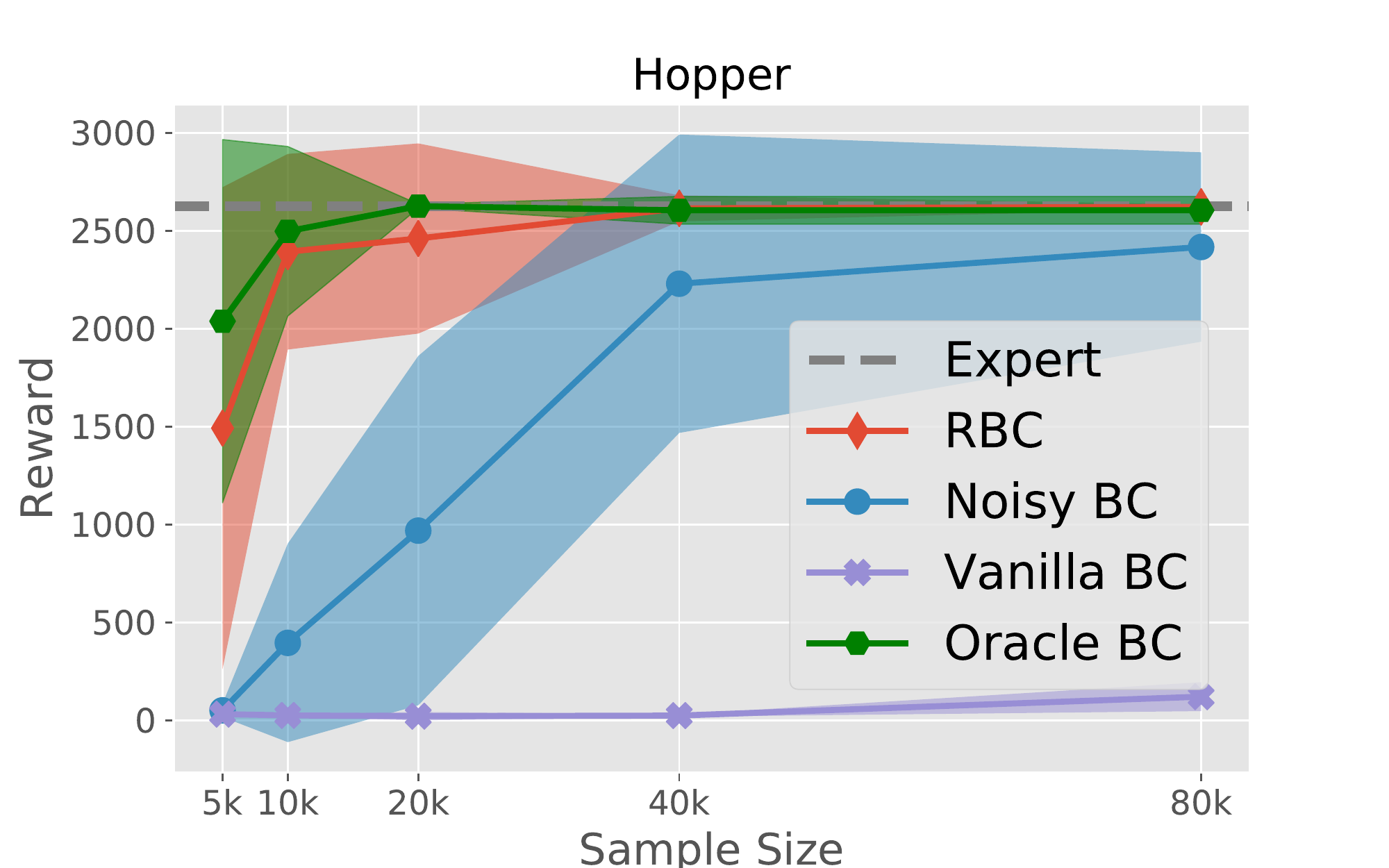}
\includegraphics[width=.45\linewidth]{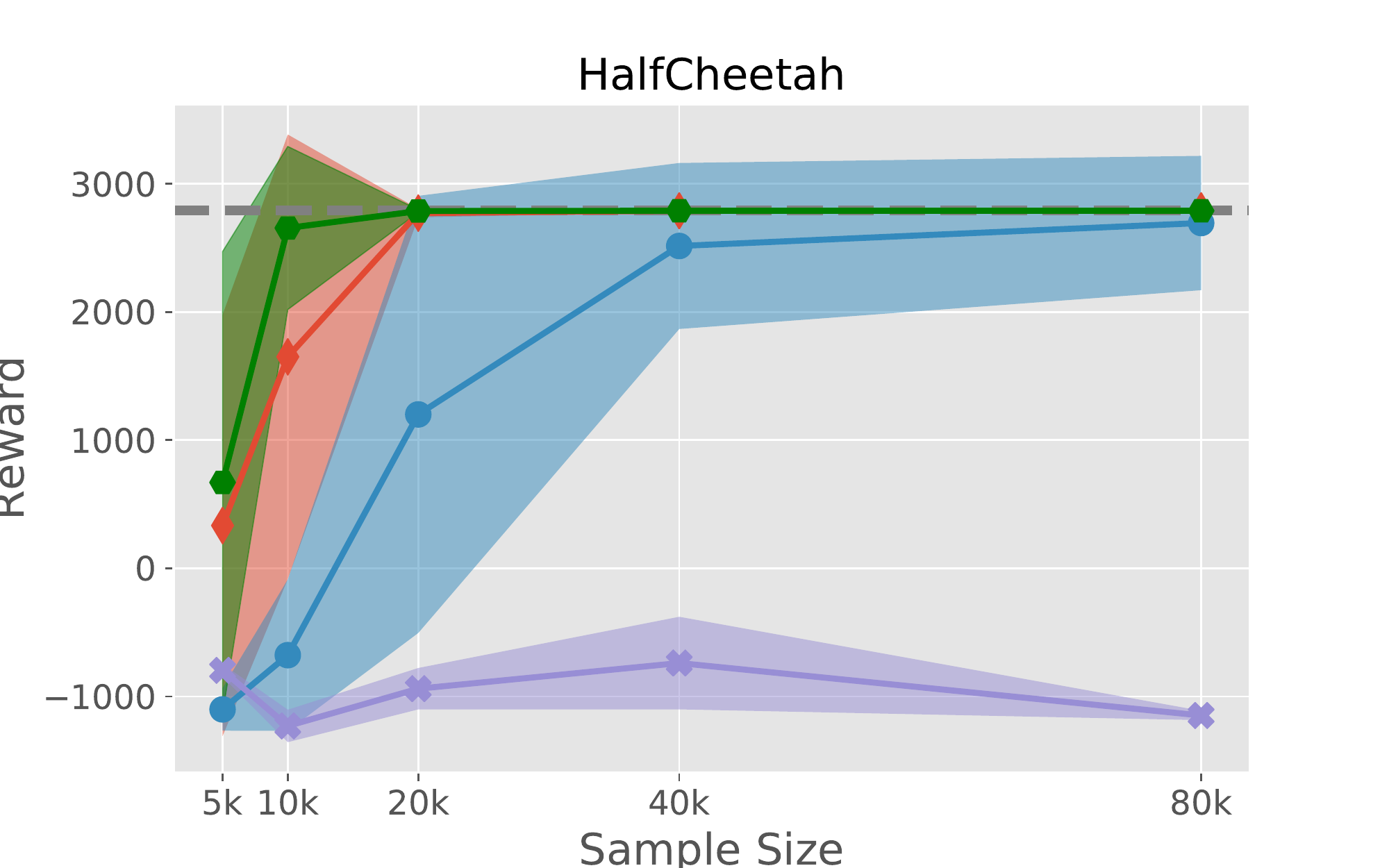}
\includegraphics[width=.45\linewidth]{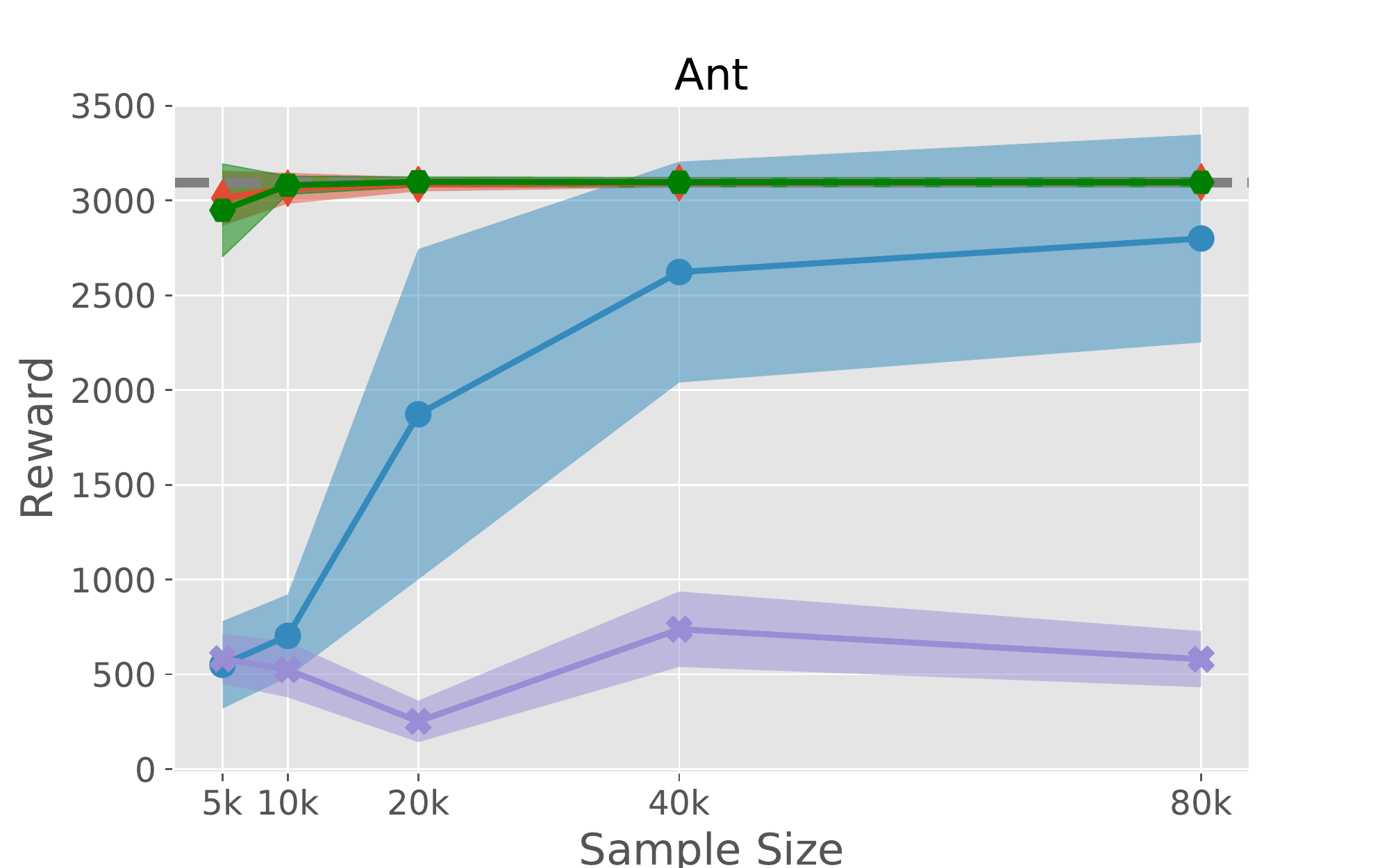}
\includegraphics[width=.45\linewidth]{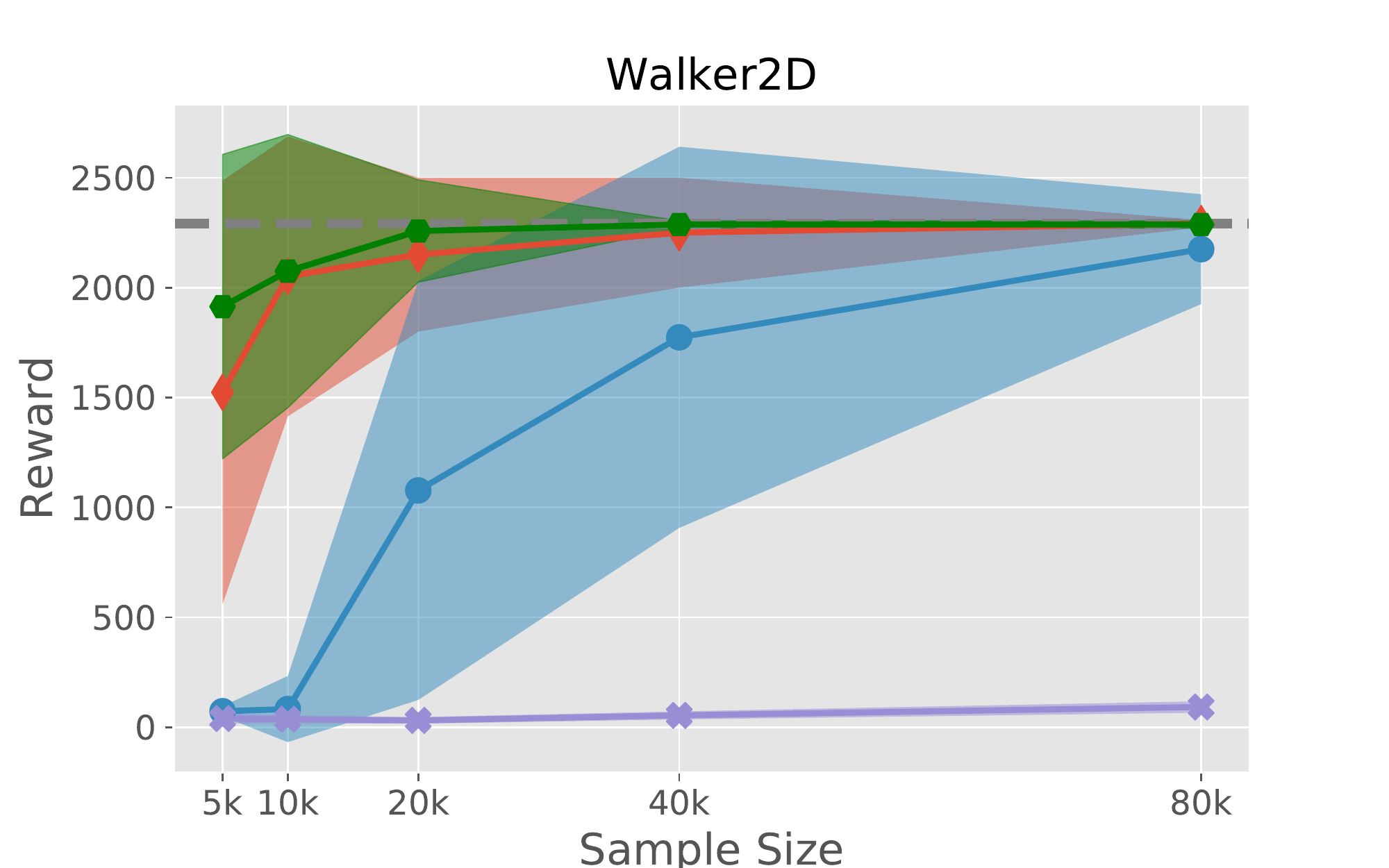}
\end{subfigure}%
\caption{Reward vs. Sample Size with fixed $\epsilon = 0.15$ in the demonstrations.
We  set the outliers randomly at the boundary -1 or 1, and vary the sample size from 5,000 to 80,000.
For different algorithms, we report  reward performance after training with 200 epochs.
The shaded region represents the standard deviation for 20 trials.
Vanilla BC (purple line) on corrupted demonstrations fails drastically, and cannot converge to expert level with large sample size.
Oracle BC (green line) denotes BC on expert demonstrations with the same sample size.
Our RBC (red line) achieves nearly the same reward performance  as Oracle BC even the sample size of the demonstration is small,
and achieves superior results compared to the state-of-the-art robust imitation learning method (blue line) \cite{NoisyBC2020}.
}
\label{fig:curve_size}
\end{figure*}

In \Cref{fig:curve} and \Cref{fig:curve_size},
our algorithm achieves superior results compared to the state-of-the-art robust imitation learning method \cite{NoisyBC2020} under different setups.
The key difference between  our method and the re-weighting idea \cite{NoisyBC2020} is that we can guarantee removing the outliers in the objective function \cref{equ:pi_MOMBC}, whereas the outliers may mislead the re-weighting process during the iterations. If the outliers have large weight in the previous iteration, the re-weighting  process will exacerbate the outliers' impacts.  
If an authentic and informative sample has large training loss, it will be down-weighted incorrectly, hence losing sample efficiency.
As in \Cref{fig:curve_size}, our RBC benefits from the \emph{tight theoretical bound} \Cref{thm:V_bound}, and achieves superior performance even the sample size (the number of trajectories) is small.



\section{Related Work}
\label{sec:related}

\textbf{Imitation Learning.} 
Behavior Cloning (BC) is the most widely-used imitation learning algorithm \citep{Pomerleau1988NIPS_BC, Pieter2018algorithmic} due to its simplicity, effectiveness and scalability, and has been widely used in practice. From a theoretical viewpoint,  it has been showed that BC achieves informational optimality in the offline setting \citep{Jiao2020FundamentalLimits} where we do not have  \emph{further online interactions} or the knowledge of the transition dynamic $\Tran$.
With online interaction, there's a line of research focusing on improving BC in different scenarios -- 
for example, 
\cite{Ross2011AIStats} proposed DAgger (Data Aggregation) by querying the expert policy in the online setting. 
\cite{brantley_SunWen_2019disagreement} proposed using an ensemble of BC as uncertainty measure and interacts with the environment to improve BC by taking the uncertainty into account, without the need to query the expert. 
Very recently, \citep{YuYang2021error, JiaoHan2021breaking, xu2021TAIL} leveraged the knowledge of the transition dynamic $\Tran$ to eliminate compounding error/distribution shift issue in BC.

Besides BC, there are other imitation learning algorithms:
\cite{ho2016GAIL} used generative adversarial networks for distribution matching to learn a reward function; 
\cite{reddy2019sqil} provided a RL framework to deal with IL by artificially setting the reward;
\cite{GuShixiang2020divergence} unified several existing imitation learning algorithm as minimizing distribution divergence between learned policy and expert demonstration, just to name a few.







\textbf{Offline RL.}
RL leverages the signal from reward function to train the policy.
Different from IL, offline RL often does not require the demonstration to be expert demonstration \citep[e.g.][]{fujimoto2019off, fujimoto2021minimalist, kumar2020conservative} (interested readers are referred to \citep{levine2020offline}), and 
even expects the offline data  with higher coverage for different sub-optimal policies \citep{buckman2020importance, jin2021pessimism, rashidinejad2021bridging}.
Behavior-agnostic setting~\citep{nachum2019dualdice, mousavi2020blackbox} even does not require the collected data from a single policy.

The closest relation between offline RL and IL is the learning of stationary visitation distribution, where learning such visitation distribution does not involve with reward signal, similar to IL.
A line of recent research especially for off-policy evaluation tries to learn the stationary visitation distribution of a given target policy~\citep[e.g.][]{liu2018breaking, nachum2019dualdice, tang2019doubly, mousavi2020blackbox, dai2020coindice}.
Especially \cite{Kostrikov2020Imitation} leverages the off-policy evaluation idea to IL area.



\textbf{Robustness in  IL and RL.}
There are several recent papers consider corruption-robust in either RL or IL.
In RL, \cite{Zhu2102} considers that the adversarial corruption may corrupt the whole episode in the online RL setting 
while a more recent one \citep{Zhu2106} considers \emph{offline RL} where $\epsilon$-fraction of the whole data set can be replaced by the outliers. However, the $\epsilon$ dependency scales with the dimension in \cite{Zhu2106}, yet $\epsilon$ can be a constant in this paper for robust offline IL. 
Many other papers consider perturbations, heavy tails, or corruptions in either reward function \citep{bubeck2013MOM} or in transition dynamic
\citep{xu2012DRO_MDP, Mannor2014scaling, HuanXu2017RLmismatch}.

The most related papers follow a similar setting of robust IL are \citep{wu2019imperfect, tangkaratt2020VILD, tangkaratt2021RobustImitation,
brown2019TREX, NoisyBC2020}, where they consider imperfect or noisy observations in imitation learning. 
However, they do not provide theoretical guarantees  
to  handle
arbitrary
outliers in the demonstrations.
And to the best of our knowledge, we provide the \emph{first
 theoretical guarantee} robust to constant fraction of arbitrary outliers in offline imitation learning.
Furthermore, 
\citep{wu2019imperfect, tangkaratt2020VILD, tangkaratt2021RobustImitation}
require additional  \emph{online interactions} 
with the environment,
and \cite{brown2019TREX, wu2019imperfect} require annotations for each demonstration, which costs a significant
human effort. 
Our algorithm achieves robustness guarantee from purely \emph{offline} demonstration, without the potentially costly or risky interaction with the real world environment or human annotations.






\subsection{Summary and Future Works} 
In this paper, we considered the corrupted demonstrations issues in imitation learning, and  proposed a novel robust algorithm, Robust Behavior Cloning, to deal with the corruptions in offline demonstration data set. 
The core technique is replacing the vanilla Maximum Likelihood Estimation with a Median-of-Means (MOM) objective which guarantees the  policy estimation and reward performance in the presence of constant fraction of outliers. 
Our algorithm has strong robustness guarantees and has competitive practical performance compared to existing methods.

There are several avenues for future work: 
since our work focuses on the corruption in offline data, 
 any improvement in \emph{online IL} which utilizes BC would benefit from the corruption-robustness guarantees and practical effectiveness
by our \emph{offline} RBC. 
Also,
it would also be of interest to apply our algorithm for real-world environment, such as automated medical diagnosis and autonomous driving.

\clearpage



\clearpage


\newpage

\appendix

\setstretch{1.3}

\section{Proofs}
\label{sec:proof}
The analysis of  maximum likelihood estimation is standard in i.i.d. setting for the supervised learning setting \citep{geer2000empirical}. 
In our proofs of the robust offline imitation learning algorithm, the analysis for the sequential decision making leverages the martingale analysis technique from   \citep{zhang2006AOS, Agarwal2020MLE}. 

Our Robust Behavior Cloning (\Cref{def:RBC}) solves the following optimization
\begin{align}
    \pir =  \arg \min_{\pi \in \Pi} \max_{\pi' \in \Pi}   \underset{1\leq j \leq M}{\mathrm{median}} 
    \left( \ell_j(\pi) - \ell_j(\pi') \right),
\end{align} 
where the loss function $\ell_j(\pi)$ is the average Negative Log-Likelihood in the batch $B_j$:
\begin{align} 
    \ell_j(\pi) = \frac{1}{b}\sum_{(\obs, \acts) \in B_j}-\log(\pi(\acts | \obs)).
\end{align}

This can be understood as a robust counterpart for the maximum likelihood estimation in sequential decision process.

With a slight abuse of notation, we use  $x_i$ and $y_i$ to denote the observation and action, and the underlying unknown expert distribution is 
$y_i \sim p(\cdot|x_i)$ and $p(y|x) = f^*(x, y)$.
Following \Cref{ass:discrete}, we have the realizable $f^* \in \Fcal$, and the discrete function class satisfies  $|\Fcal| < \infty$.

Let $\Dcal$ denote the data set and let $\Dcal'$ denote a tangent sequence
$\{x_i', y_i'\}_{i=1}^{|\Dcal|}$.
The tangent sequence is defined as $x_i' \sim \Dist_i(x_{1:i-1}, y_{1:i-1})$ and $y_i' \sim p(\cdot|x_i')$. Note here that $x_i'$ follows from the distribution $\Dist_i$, and depends on the original sequence, hence the tangent sequence is independent conditional on $\Dcal$. 

For this martingale process, we first introduce a decoupling Lemma from \cite{Agarwal2020MLE}. 

\begin{lemm}
\label{lem:decouple}[Lemma 24 in \cite{Agarwal2020MLE}]
Let $\Dcal$ be a dataset, and let $\Dcal'$ be a tangent sequence. 
Let $\Gamma(f,\Dcal) = \sum_{(x,y)\in \Dcal}\phi(f,(x,y))$ be any function
which can be decomposed additively across samples in $\Dcal$.
Here, $\phi$ is any
function of $f$ and sample $(x, y)$.
Let $\fhat =\fhat(\Dcal)$ be any estimator taking  the dataset $\Dcal$ as input and with range $\Fcal$. Then we have
\begin{align*}
\Expe_{\Dcal}\left[   \exp \left(
\Gamma(\fhat,\Dcal) - \log \Expe_{\Dcal'}\exp(\Gamma(\fhat,\Dcal')) - \log |\Fcal|   \right)   \right] \leq 1.
\end{align*}
\end{lemm}

Then we present a Lemma which upper bounds the TV distance via a loss function closely related to KL divergence. Such bounds for probabilistic distributions are discussed extensively in literature such as
\cite{Tsybakov2009IntroductionTN}.
\begin{lemm}
\label{lem:strongconvexity}
[Lemma 25 in \cite{Agarwal2020MLE}]
For any two conditional probability densities $f_1,f_2$ and any state
distribution $\Dist \in \Delta(\mathcal{X})$ we have
\begin{align*}
\Expe_{x \sim \Dist} \normtv{f_1(x,\cdot) - f_2(x,\cdot)}^2 \leq - 2\log \Expe_{x\sim\Dist,y \sim f_2(\cdot \mid x)} \exp \left(- \frac{1}{2}\log\frac{f_2(x,y)}{f_1(x,y)} \right).
\end{align*}
\end{lemm}

\subsection{Proof of \Cref{thm:pi_bound}}

With these Lemmas in hand, we are now equipped to prove our main theorem (\Cref{thm:pi_bound}), which guarantees the solution $\pir$ of \cref{equ:pi_MOMBC} is close to the optimal policy $\pis$ in TV distance.
\begin{theo}
[\Cref{thm:pi_bound}]
Suppose we have corrupted demonstration data set  $\Demo$ with sample size $N$ from 
\Cref{def:Huber}, and there exists a constant corruption ratio $\epsilon < 0.5$.
Under \Cref{ass:discrete}, 
let $\tau$ to be the output objective value with $\pir$ in the optimization \cref{equ:pi_MOMBC} with the batch size $b \leq  \frac{1}{3 \epsilon}$, then 
with probability at least $1-c_1\delta$,
we have
\begin{align*}
    \Expe_{s \sim \dist} \normtv{ \pir(\cdot|\obs) - \pis(\cdot|\obs)}^2 = O \left( \frac{\log(|\Pi|/\delta)}{N} + \tau \right).
\end{align*}
\end{theo}

\begin{proof}[Proof of \Cref{thm:pi_bound}]
En route to the proof of \Cref{thm:pi_bound}, we keep using the notations in \Cref{lem:decouple} and \Cref{lem:strongconvexity}, where the state observation is $x$, the action is $y$, and the discrete function class is $\Fcal$.

Similar to \cite{Agarwal2020MLE}, we first note that \Cref{lem:decouple} can be combined with a simple Chernoff
bound to obtain an exponential tail bound. With probability at least $1-c_1\delta$, we have
\begin{align}\label{equ:chernoff}
-\log \Expe_{\Dcal'} \exp(\Gamma(\fhat, \Dcal')) \leq -\Gamma(\fhat,\Dcal) + \log |\Fcal| + \log(1/\delta).
\end{align}
Our proof technique relies on lower bounding the LHS of \cref{equ:chernoff}, and upper bounding the RHS \cref{equ:chernoff}.

Let the batch size $b \leq \frac{1}{3\epsilon}$, which is  a constant in \Cref{def:RBC}, then the number of batches $M \geq 3\epsilon N$ such that there exists at least 66\% batches without corruptions. 

In the definition of RBC (\Cref{def:RBC}), we solve
\begin{align} \label{equ:pi_MOMBC_append}
    \pir =  \arg \min_{\pi \in \Pi} \max_{\pi' \in \Pi}   \underset{1\leq j \leq M}{\mathrm{median}} 
    \left(\ell_j(\pi) - \ell_j(\pi')\right).
\end{align} 

Notice that since $\pis$ is one feasible solution of the inner maximization step \cref{equ:pi_MOMBC_append}, we can choose $\pi' = \pis$. Now we consider the objective function which is the difference of Negative Log-Likelihood between $f$ and $f^*$, i.e.,   $\ell_j(f) - \ell_j(f^*)$, defined in \cref{equ:loss} where
\begin{align*} 
    \ell_j(\pi) = \frac{1}{b}\sum_{(\obs, \acts) \in B_j}-\log(\pi(\acts | \obs)).
\end{align*}

Hence, we choose $\Gamma(f,\Dcal)$ in \Cref{lem:decouple} as
\begin{align*}
    \Gamma_j(f, \Dcal) &= \frac{N}{b}\sum_{i\in B_j} -\frac{1}{2}
\log \frac{f^*(x_i,y_i)}{f(x_i,y_i)} \\
&= \frac{N}{2b}\sum_{i\in B_j}
\left(
\log f(x_i,y_i) - 
\log f^*(x_i,y_i)  
\right),
\end{align*}
which is the difference of Negative Log-Likelihood  ${N}(\ell_j(f^*) - \ell_j(f))/2$ evaluated on a single batch $B_j, j \in [M]$. This is actually the objective function on a single batch appeared in \cref{equ:pi_MOMBC}.

\paragraph{Lower bound for the LHS of \cref{equ:chernoff}.}
We apply the concentration bound \cref{equ:chernoff}  for such uncorrupted batches, hence the majority of all batches satisfies  \cref{equ:chernoff}.
For those batches, the LHS of \cref{equ:chernoff} can be lower bounded by the TV distance according to  \Cref{lem:strongconvexity}.
\begin{align}
& - \log \Expe_{\Dcal'} \left[ \exp \left( \frac{N}{b}\sum_{i\in B_j} -\frac{1}{2}\log\left( \frac{f^\star(x_i',y_i')}{\fhat(x_i',y_i')} \right) \right) \big| \Dcal \right] \nonumber \\
& \overset{(i)}{=} - \frac{N}{b}\sum_{i\in B_j}\log \Expe_{x,y\sim \Dist_i} \exp \left(-\frac{1}{2} \log \frac{f^\star(x,y)}{\fhat(x,y)} \right) \nonumber \\
& \overset{(ii)}{\geq} \frac{N}{2b}\sum_{i\in B_j} \Expe_{x \sim \Dist_i} \normtv{\fhat(x,\cdot) - f^\star(x,\cdot)}^2, \label{equ:lower}
\end{align}
where (i) follows from the independence between $\fhat$ and $\Dcal'$ due to the decoupling technique, and (ii) follows from \Cref{lem:strongconvexity}, which is an upper bound
of the Total Variation distance.

\paragraph{Upper bound for the RHS of \cref{equ:chernoff}.}
Note that the objective is the median of means of each batches and  $f^*$ is one feasible solution of the inner maximization step \cref{equ:pi_MOMBC_append}.
Since $\tau$ is the output objective value with $\pir$ in the optimization \cref{equ:pi_MOMBC}, this implies that
$\ell_{\rm Med}(\pi) - \ell_{\rm Med}(\pi') \leq \tau$ for the median batch $B_{\rm Med}$, which is equivalent to
$-\Gamma_{\rm Med}(f, \Dcal) \leq N\tau / 2$.

Hence for the median batch $B_{\rm Med}$,  the RHS of \cref{equ:chernoff} can be upper bounded by
\begin{align}\label{equ:upper}
-\Gamma_{\rm Med}(\fhat,\Dcal) + \log |\Fcal| + \log(1/\delta) 
\leq \log |\Fcal| + \log(1/\delta) + N \tau / 2.
\end{align}

Putting together the pieces \cref{equ:lower} and \cref{equ:upper} for $B_{\rm Med}$, we have
\begin{align*}
    \Expe_{s \sim \dist} \normtv{ \pir(\cdot|s) - \pis(\cdot|s)}^2 = O \left( \frac{\log(|\Fcal|/\delta)}{N} + \tau \right),
\end{align*}
with probability at least $1 - c_1\delta$.

\end{proof}

\subsection{Proof of \Cref{thm:V_bound}}

With the supervised learning guarantees \Cref{thm:pi_bound} in hand, which provides an upper bound for $\Expe_{s \sim \dist} \normtv{ \pir(\cdot|s) - \pis(\cdot|s)}^2$, we are now able to present the suboptimality guarantee of the reward for $\pir$. 
This bound directly corresponds to the reward performance of a policy. 
\begin{theo}[\Cref{thm:V_bound}]
Under the same setting as \Cref{thm:pi_bound}, we have
\begin{align*}
  J_{\pi_E} - J_{\pir} \leq
  O \left( \frac{1}{(1 - \gamma)^2}    \sqrt{\frac{\log(|\Fcal|/\delta)}{N} + \tau } \right),
\end{align*}
with probability at least $1 - c_1\delta$.
\end{theo}

\begin{proof}[Proof of \Cref{thm:V_bound}]
This part is similar to \cite{Agarwal2019AJKS}, and we have
\begin{align*}
    (1 - \gamma) (J_{\pi_E} - J_{\pir}) &= \Expe_{s\sim \dist} \Expe_{a\sim {\pis(\cdot | s)}}  A^{\pir}(\obs, \acts)\\
    &\leq \frac{1}{1 - \gamma } \sqrt{\Expe_{s\sim \dist} \norm{\pir(\cdot|s) - \pis(\cdot|s)}_1^2}\\
    &=    \frac{2}{1 - \gamma } \sqrt{\Expe_{s\sim \dist} \normtv{\pir(\cdot|s) - \pis(\cdot|s)}^2},
\end{align*}
where we use the fact that
$\sup_{s, a, \pi} \abs{A^\pi(\obs, \acts)} \leq \frac{1}{1-\gamma}$ for the advantage function and the reward is always bounded between 0 and 1.

Combining \Cref{thm:pi_bound}, we have
\begin{align*}
  J_{\pi_E} - J_{\pir} \leq
  O \left( \frac{1}{(1 - \gamma)^2} 
   \sqrt{\frac{\log(|\Fcal|/\delta)}{N} + \tau } \right),
\end{align*}
with probability at least $1 - c_1\delta$.
\end{proof}

\section{Experimental Details}

\label{sec:Exp_appendix}

In this section, we provide the details of our algorithm RBC in different setups.
All Behavior Cloning models were trained to minimize the mean-squared error regression loss on the demonstration data for 200 epochs using Adam \cite{kingma2014adam}.
In all setting, we fix the policy network as 3 hidden layer feed-forward Neural Network of size $\{500, 500, 500\}$ with ReLU activation. More hyperparameters are provided in \Cref{tab:para}.

\begin{table}[h!]
\centering
\caption{\label{tab:para}Hyperparameters}
\begin{tabular}{|l|l|}
\hline
\textbf{Hyperparameter}        & \textbf{Value}                      \\ \hline \hline
Parallel Environments & 20                         \\ \hline
$\ell_2$ regularization & 0                         \\ \hline
Entropy coefficient & 0.01                         \\ \hline
Gradient clipping     & 0.1                        \\ \hline
Learning rate         & $7.5*10^{-4}$ \\ \hline
\end{tabular}
\end{table}

\paragraph{Reward vs. Epochs.}
We illustrate the convergence of our algorithm by tracking the reward performance 
for different continuous control environments
simulated by PyBullet \cite{coumans2016pybullet} simulator: HopperBulletEnv-v0, 
Walker2DBulletEnv-v0,
HalfCheetahBulletEnv-v0
and AntBulletEnv-v0.
More specifically, we evaluate current policy in the simulator for 20 trials, and obtain the mean and standard deviation of cumulative reward for every 5 epochs.
In this experiment, we adopt option (1) for the outliers, which set the actions of outliers to the boundary ($-1$ or $+1$).
In \Cref{fig:epochs_appendix}, we
fix the corruption ratio as 10\% and 20\% with
fixed demonstration
data of size 60000, and present the Reward vs. Epochs.
We note that the difference of purple curves in  $\epsilon=0.1$ and $\epsilon=0.2$ is due to different random seed.

\paragraph{Convergence time.}
Another important aspect of our algorithm is the practical time complexity efficiency.
To speed up our RBC, 
we pick multiple batches around the median batch in line 7 of \Cref{alg:RBC}, and evaluate the gradient using back-propagation on these batches.
The experimental setup is consistent with \Cref{fig:epochs_appendix}.
To directly compare the time complexity, we report the reward vs. wall clock time performance of our RBC and ``Oracle BC'', which optimizes on the expert demonstrations.

We measure the convergence time by
counting the elapsed time  from zero to first achieving 95\% of expert level.
The experiments are conducted on 1/2 core of NVIDIA T4 GPU, and presented in \Cref{tab:time},
which shows that
the actual running time time of RBC is comparable to vanilla BC.

\begin{table}[h!]
\centering
\caption{\label{tab:time}Convergence time (wall-clock seconds) }
\begin{tabular}{|l|l|l|l|l|}
\hline
          & Hopper & HalfCheetah & Ant & Walker2D \\ \hline \hline
Oracle BC & 88     & 174         & 49     &  159        \\ \hline
RBC       & 368    & 597         & 134     &   385       \\ \hline
\end{tabular}
\end{table}

\begin{figure}[t]
\centering
\begin{subfigure}{.45\columnwidth}
  \centering
  \includegraphics[width=.8\linewidth]{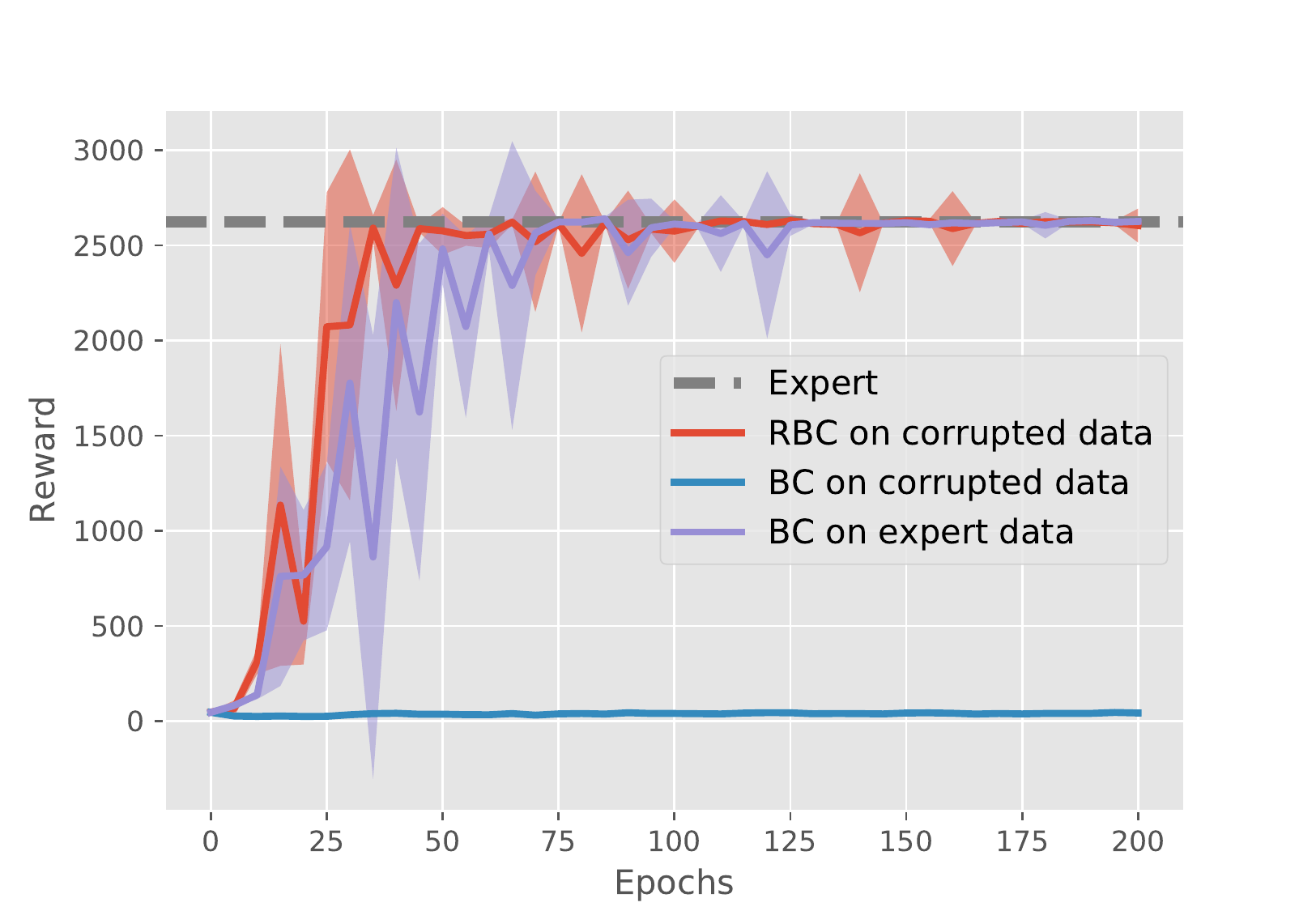}
  \caption{Hopper environment with $\epsilon = 0.1$.}
\end{subfigure}%
\begin{subfigure}{.45\columnwidth}
  \centering
  \includegraphics[width=.8\linewidth]{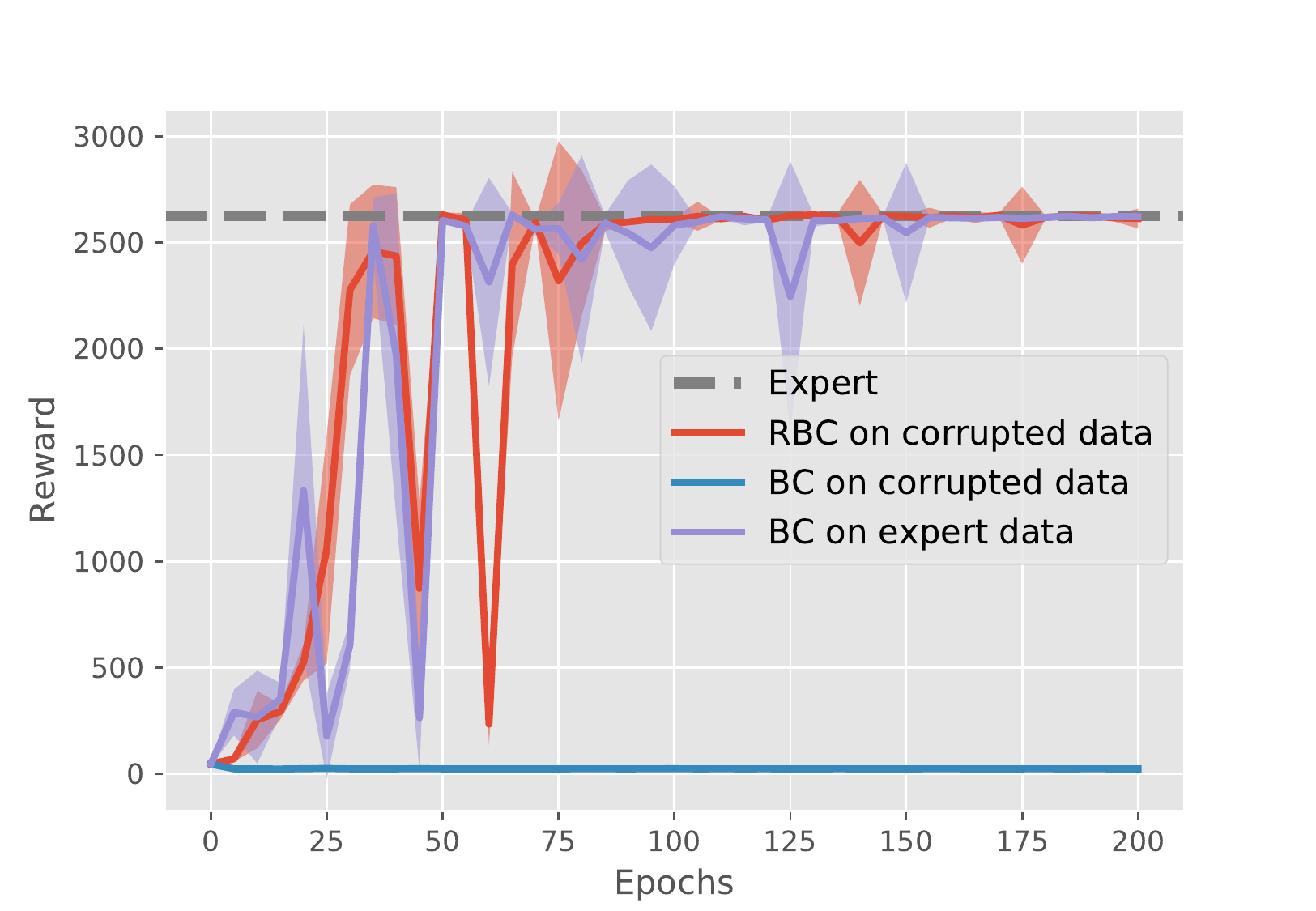}
  \caption{Hopper environment with $\epsilon = 0.2$.}
\end{subfigure}%

\begin{subfigure}{.45\columnwidth}
  \centering
  \includegraphics[width=.8\linewidth]{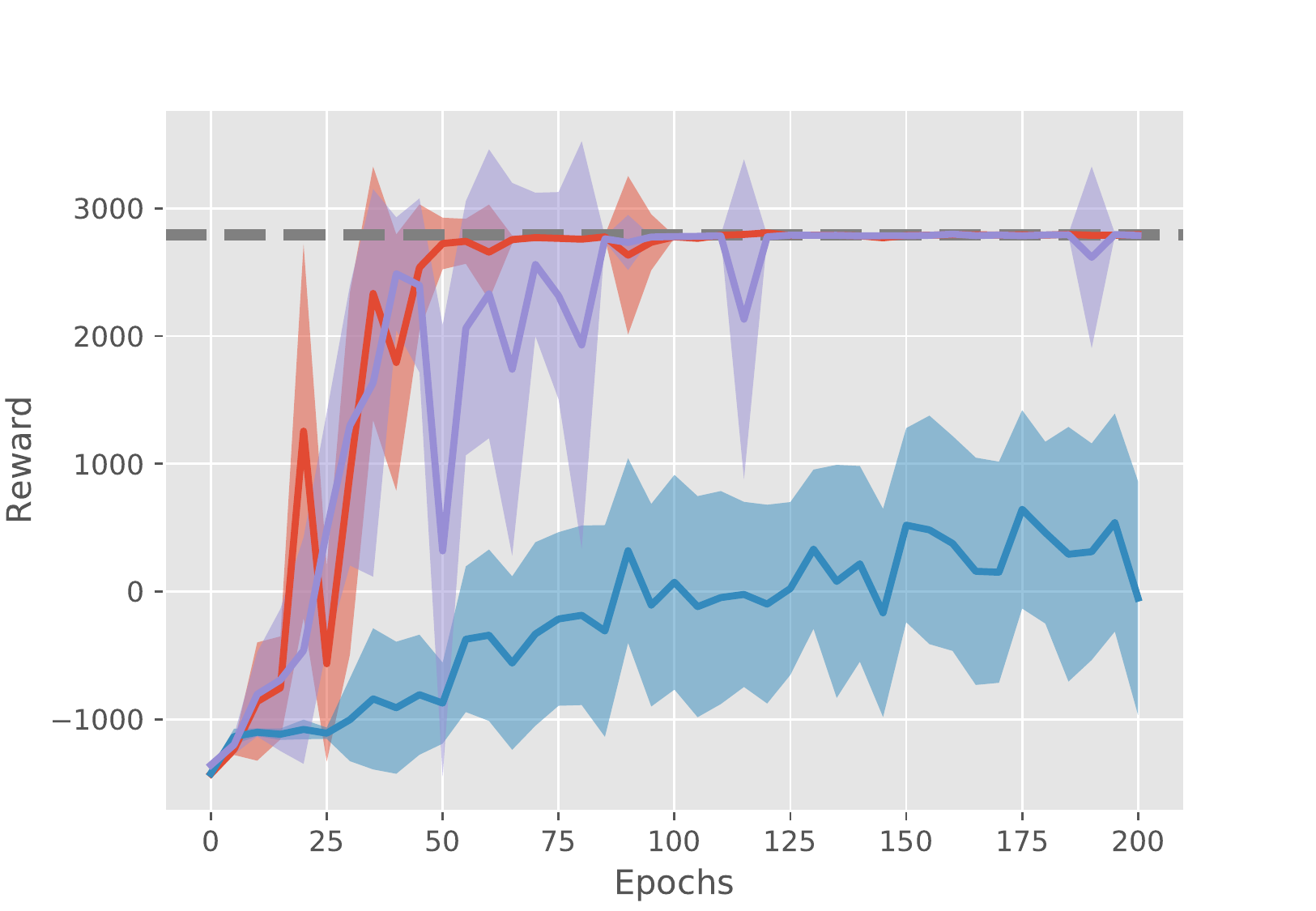}
    \caption{HalfCheetah environment with $\epsilon = 0.1$.}
\end{subfigure}
\begin{subfigure}{.45\columnwidth}
  \centering
  \includegraphics[width=.8\linewidth]{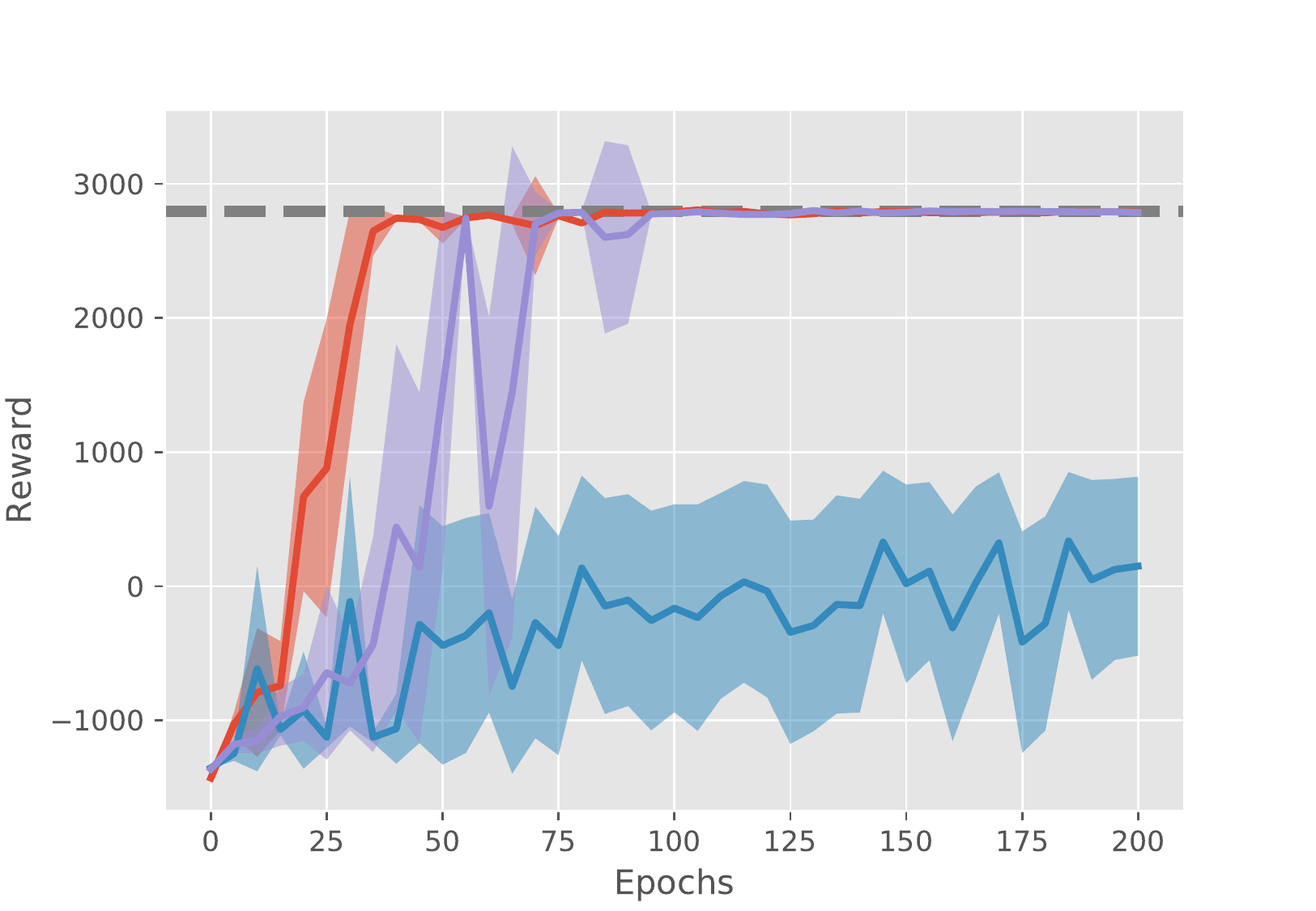}
    \caption{HalfCheetah environment with $\epsilon = 0.2$.}
\end{subfigure}

\begin{subfigure}{.45\columnwidth}
  \centering
  \includegraphics[width=.8\linewidth]{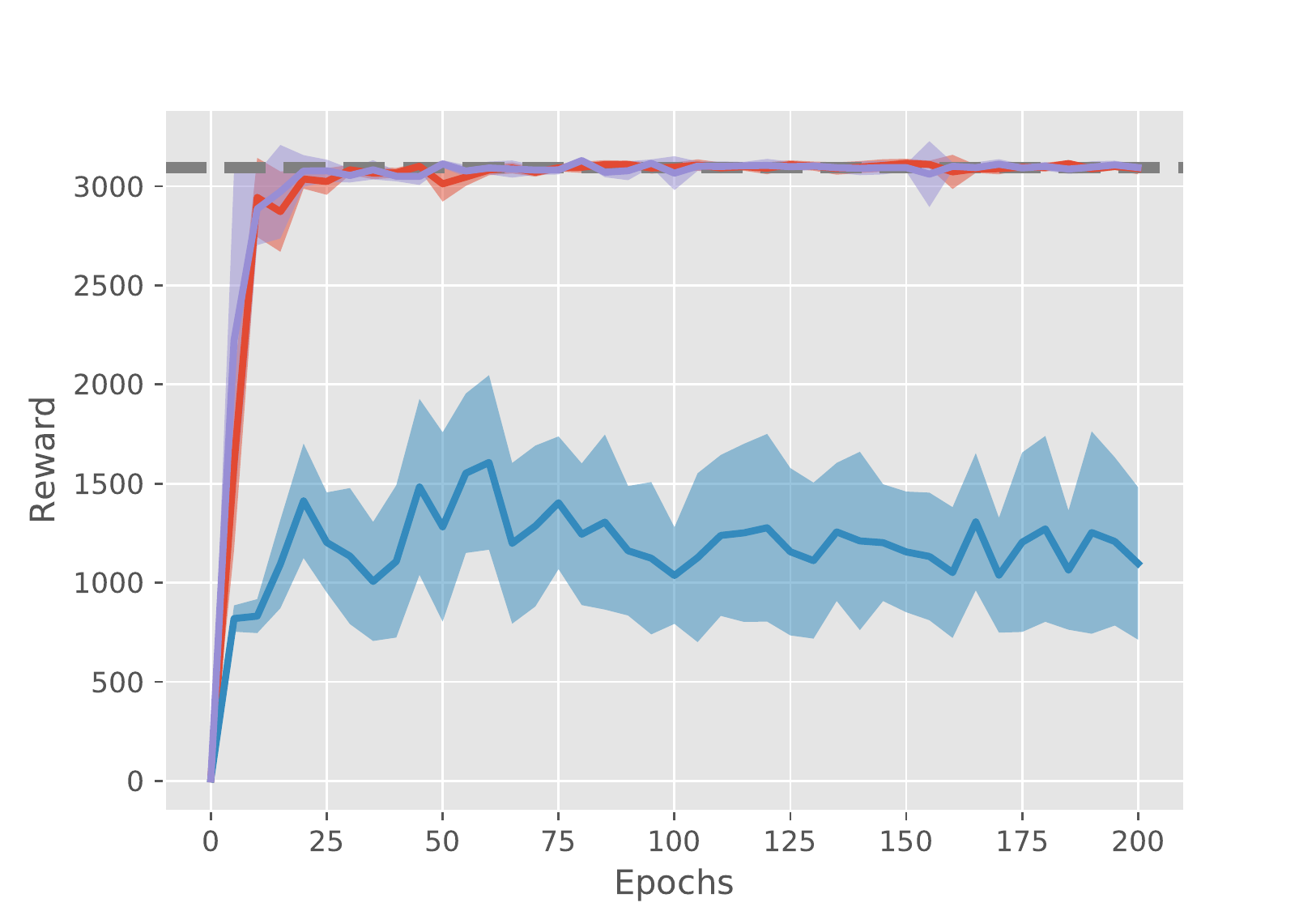}
      \caption{Ant environment with $\epsilon = 0.1$.}
\end{subfigure}
\begin{subfigure}{.45\columnwidth}
  \centering
  \includegraphics[width=.8\linewidth]{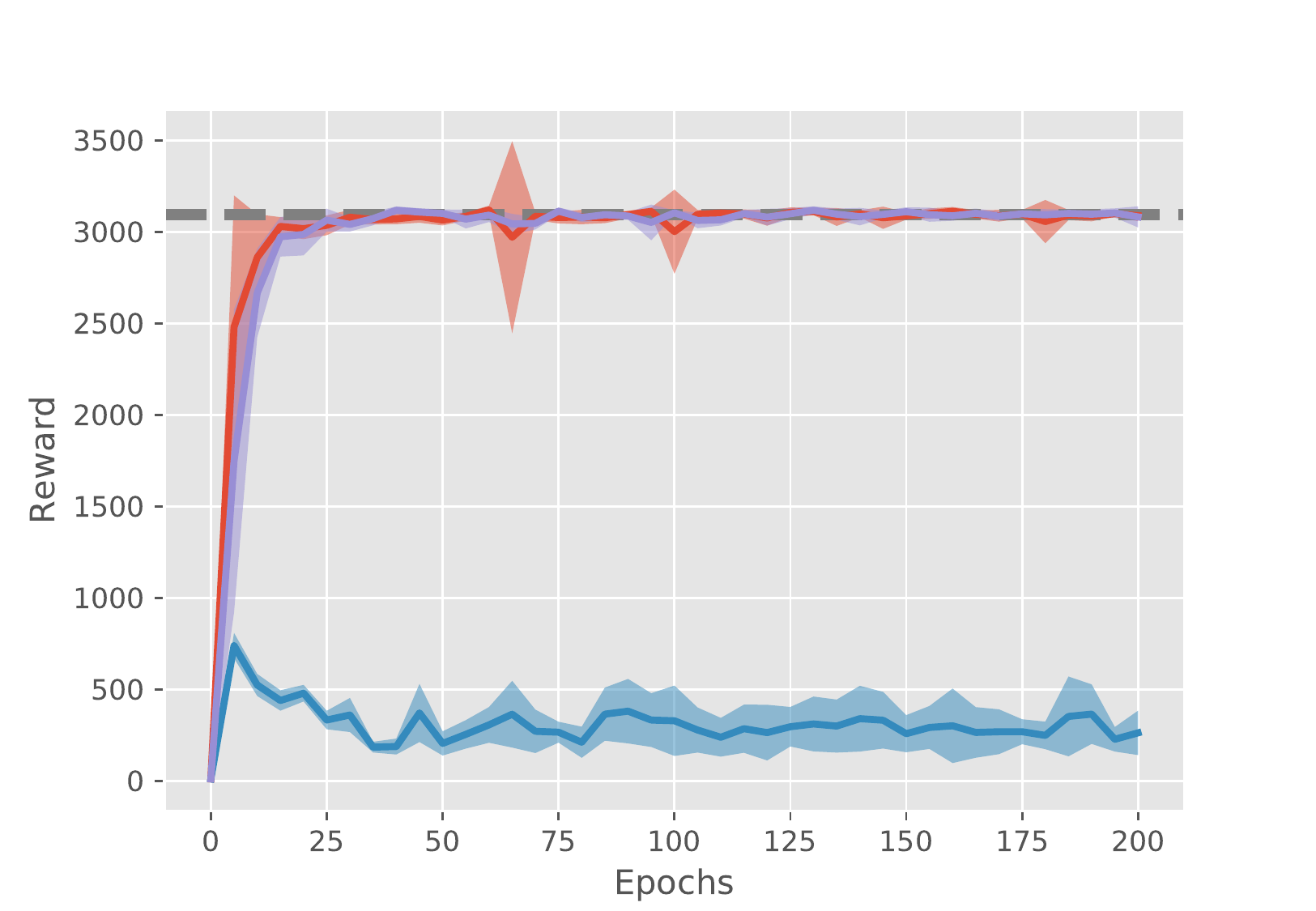}
      \caption{Ant environment with $\epsilon = 0.2$.}
\end{subfigure}

\begin{subfigure}{.45\columnwidth}
  \centering
  \includegraphics[width=.8\linewidth]{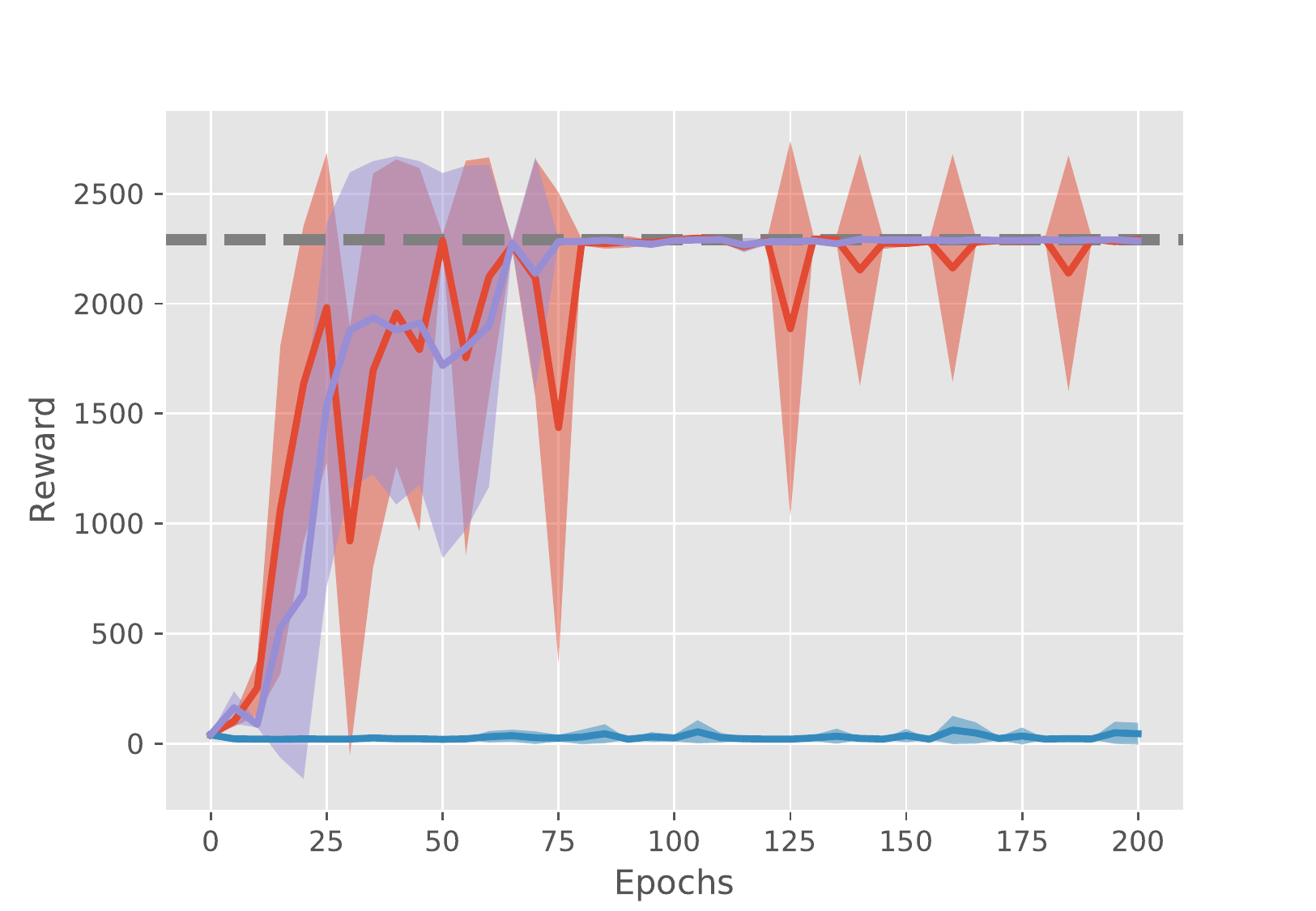}
      \caption{Walker2D environment with $\epsilon = 0.1$.}
\end{subfigure}
\begin{subfigure}{.45\columnwidth}
  \centering
  \includegraphics[width=.8\linewidth]{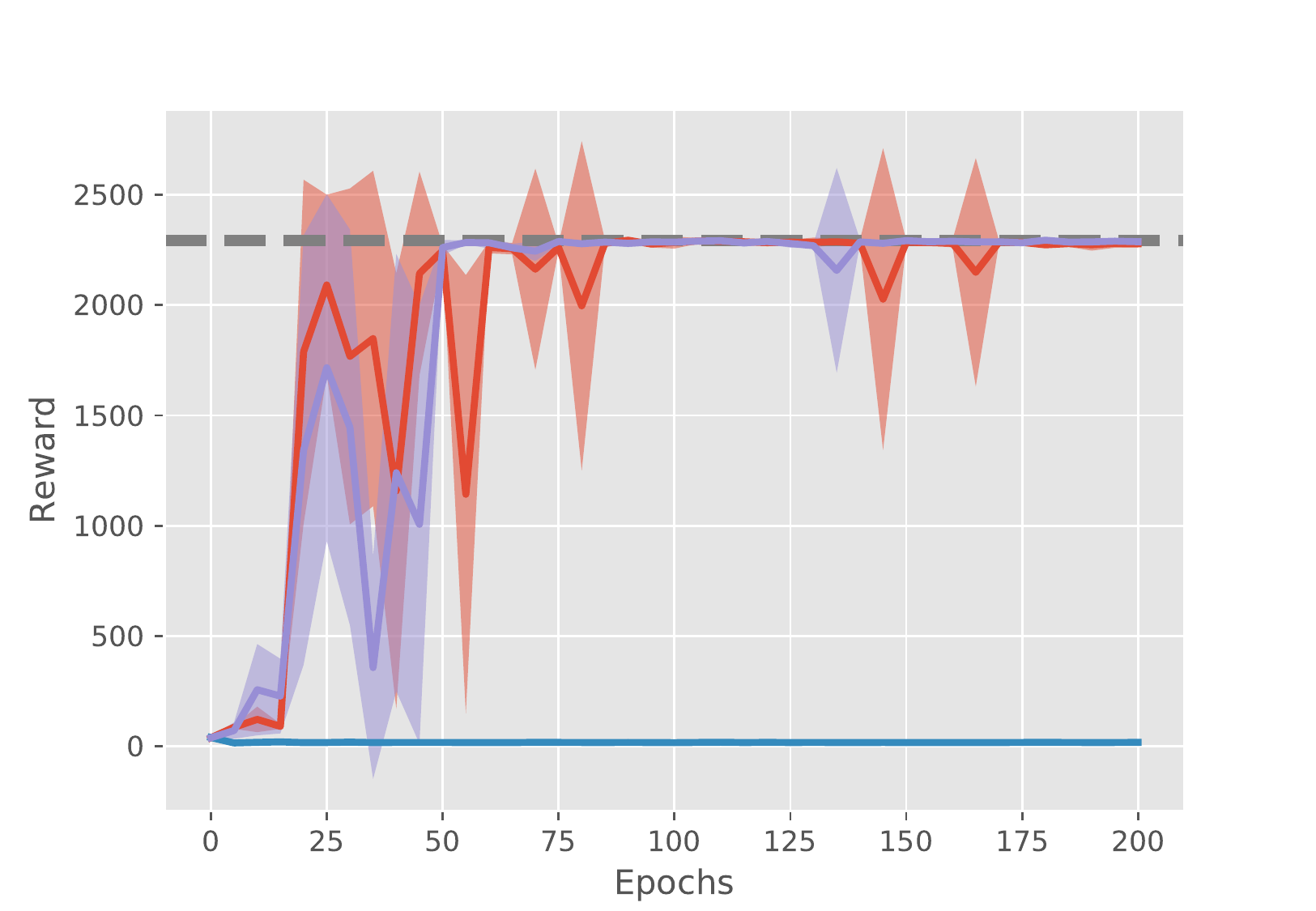}
      \caption{Walker2D environment with $\epsilon = 0.2$.}
\end{subfigure}
\caption{Reward vs. Epochs for 
offline Imitation Learning on four different continuous control tasks from PyBullet \cite{coumans2016pybullet} with
fixed demonstration
data of size 60000. 
We choose the corruption ration $\epsilon =$ 10\%, 20\%. 
For every 5 epochs, we evaluate
the current policy in the environment for 20 trials, and the shaded region represents the standard
deviation.
We note that the difference of purple curves between left and right is due to different random seed.
Vanilla BC on corrupted demonstrations fails to converge to expert policy. 
Using the
robust counterpart \Cref{alg:RBC} on corrupted demonstrations has good convergence properties. Surprisingly, our RBC on corrupted demonstrations has nearly the same reward performance of using
BC on \emph{expert demonstrations}.
}
\label{fig:epochs_appendix}
\end{figure}

\end{document}